\documentclass{article}
\usepackage{iclr2022_conference,times}

\usepackage{amsmath}

\usepackage[utf8]{inputenc} %
\usepackage{lipsum} %

\usepackage[T1]{fontenc}

\usepackage{etoolbox}
\newbool{includeappendix}
\setbool{includeappendix}{true}

\ifdefined\isoverfull
\else
\fi

\usepackage{acro} %

\DeclareAcronym{cli} {
    short = CLI,
    long = Command Line Interface,
    class = abbrev
}

\usepackage{listings}

\usepackage{textcomp}

\usepackage{xcolor}

\usepackage[scaled=0.8]{beramono}

\definecolor{ckeyword}{HTML}{7F0055}
\definecolor{ccomment}{HTML}{3F7F5F}
\definecolor{cstring}{HTML}{2A0099}

\lstdefinestyle{numbers}{
	numbers=left,
	framexleftmargin=20pt,
	numberstyle=\tiny,
	firstnumber=auto,
	numbersep=1em,
	xleftmargin=2em
}

\lstdefinestyle{layout}{
	frame=none,
	captionpos=b,
}

\lstdefinestyle{comment-style}{
	morecomment=[l]//,
	morecomment=[s]{/*}{*/},
	commentstyle={\color{ccomment}\itshape},
}

\lstdefinestyle{string-style}{
	morestring=[b]",%
	morestring=[b]',%
	stringstyle={\color{cstring}},
	showstringspaces=false,%
}

\lstdefinestyle{keyword-style}{
	keywordstyle={\ttfamily\bfseries},
	morekeywords={
		function,
		constructor,
		int,
		bool,
		return,
		returns,
		uint
	},
	morekeywords = [2]{},
	keywordstyle = [2]{\text},
	sensitive=true,
}

\lstdefinestyle{input-encoding}{
	inputencoding=utf8,
	extendedchars=true,
	literate=
	{ℝ}{$\reals$}1%
	{→}{$\rightarrow$}1%
	{α}{$\alpha$}1%
	{β}{$\beta$}1%
	{λ}{$\lambda$}1%
	{θ}{$\theta$}1%
	{ϕ}{$\phi$}1%
}

\lstdefinestyle{escaping}{
	moredelim={**[is][\color{blue}]{\%}{\%}},
	escapechar=|,
	mathescape=true
}

\lstdefinestyle{default-style}{
	basicstyle=\fontencoding{T1}\ttfamily\footnotesize,
	style=numbers,
	style=layout,
	style=comment-style,
	style=string-style,
	style=keyword-style,
	style=input-encoding,
	style=escaping,
	tabsize=2,
	upquote=true
}

\lstdefinelanguage{BASIC}{
	language=C++,
	style=default-style
}[keywords,comments,strings]%

\usepackage[capitalize]{cleveref}

\crefformat{section}{\S#2#1#3}

\crefrangeformat{section}{\S#3#1#4\crefrangeconjunction\S#5#2#6}

\crefmultiformat{section}{\S#2#1#3}{\crefpairconjunction\S#2#1#3}{\crefmiddleconjunction\S#2#1#3}{\creflastconjunction\S#2#1#3}

\newcommand{\crefrangeconjunction}{--}

\crefname{listing}{Lst.}{listings}
\crefname{line}{Lin.}{Lin.}
\crefname{appendix}{App.}{App.}

\newcommand{\app}[1]{%
	\ifbool{includeappendix}{\cref{#1}}{the appendix}%
}
\newcommand{\App}[1]{%
	\ifbool{includeappendix}{\cref{#1}}{The appendix}%
}

\usepackage{url}            %
\usepackage{booktabs}       %
\usepackage{amsfonts}       %
\usepackage{nicefrac}       %
\usepackage{microtype}      %

\usepackage{graphicx}
\usepackage{amsmath,amsfonts,bm}
\usepackage{amsthm}
\usepackage{thmtools,thm-restate}
\usepackage{adjustbox}
\usepackage{dsfont}
\usepackage{multirow}
\usepackage{threeparttable}
\usepackage{makecell}
\usepackage{xspace}
\usepackage{caption}
\usepackage{subcaption}
\usepackage{wrapfig}
\usepackage{tikz}
\usepackage{lipsum}
\usetikzlibrary{decorations}
\usepackage{enumitem}

\usepackage{amsmath,amsfonts,bm}

\def\eqref#1{equation~\ref{#1}}

\def\1{\bm{1}}

\def\rs{{\textnormal{s}}}

\def\vf{{\bm{f}}}

\def\vx{{\bm{x}}}

\def\mI{{\bm{I}}}

\DeclareMathAlphabet{\mathsfit}{\encodingdefault}{\sfdefault}{m}{sl}
\SetMathAlphabet{\mathsfit}{bold}{\encodingdefault}{\sfdefault}{bx}{n}

\newcommand{\R}{\mathbb{R}}

\newcommand{\softmax}{\mathrm{softmax}}

\DeclareMathOperator*{\argmax}{arg\,max}

\theoremstyle{plainnat} %
\newtheorem{theorem}{Theorem}[section]
\newtheorem*{theorem*}{Theorem}

\newcommand{\nclass}{m}

\newcommand{\bc}[1]{\mathcal{#1}}

\newcommand{\prob}{\mathcal{P}}

\newcommand{\RNM}{\texttt{ResNet50}\xspace}
\newcommand{\RNB}{\texttt{ResNet110}\xspace}
\newcommand{\EN}{\texttt{EfficientNet-B7}\xspace}
\newcommand{\LANET}{\texttt{LaNet}\xspace}

\newcommand{\cifar}{CIFAR10\xspace}
\newcommand{\IN}{ImageNet\xspace}

\newcommand{\tool}{\aces}
\newcommand{\aces}{\textsc{Aces}\xspace}
\newcommand{\ace}{\textsc{Ace}\xspace}

\newcommand{\smoothadv}{\textsc{SmoothAdv}\xspace}
\newcommand{\macer}{\textsc{Macer}\xspace}
\newcommand{\consistency}{\textsc{Consistency}\xspace}
\newcommand{\cohen}{\textsc{Gaussian}\xspace}

\newcommand{\abstain}{{\ensuremath{\oslash}}\xspace}

\usepackage{algorithm} %
\usepackage[noend]{algpseudocode} %

\iclrfinalcopy

\title{Robust and Accurate -- Compositional \\Architectures for Randomized Smoothing}
\author{%
	Miklós Z. Horváth, Mark Niklas Müller, Marc Fischer, Martin Vechev\\
	Department of Computer Science\\
	ETH Zurich, Switzerland\\
	\texttt{mihorvat@ethz.ch}, \texttt{\{mark.mueller,marc.fischer,martin.vechev\}@inf.ethz.ch}
}

\begin{document}

\maketitle

\begin{abstract}

Randomized Smoothing (RS) is considered the state-of-the-art approach to obtain certifiably robust models for challenging tasks.
However, current RS approaches drastically decrease standard accuracy on unperturbed data, severely limiting their real-world utility.
To address this limitation, we propose a compositional architecture, \tool, which certifiably decides on a per-sample basis whether to use a smoothed model yielding predictions with guarantees or a more accurate standard model without guarantees. 
This, in contrast to prior approaches, enables both high standard accuracies \emph{and} significant provable robustness.
On challenging tasks such as \IN, we obtain, e.g., $80.0\%$ natural accuracy and $28.2\%$ certifiable accuracy against $\ell_2$ perturbations with $r=1.0$.
We release our code and models at \url{https://github.com/eth-sri/aces}.

\end{abstract}

\section{Introduction}
\label{sec:introduction}
\vspace{-1.0mm}

Since the discovery of imperceptible input perturbations that can fool machine learning models, called adversarial examples \citep{BiggioCMNSLGR13,szegedy2013intriguing}, certifying model robustness has been identified as an essential task to enable their application in safety-critical domains.

Various works have discussed the fundamental trade-off between robustness and accuracy in the empirical setting \citep{Raghunathan19AdvCanHurt,TsiprasSETM19,zhang2019theoretically}. However, in the setting of deterministically certified robustness, this Pareto frontier has only recently been explored \citep{mueller2021certify}.
There, due to the poor scaling of deterministic methods to large networks, performance on more challenging tasks is severely limited.
In the probabilistic certification setting, recent works aim to jointly increase robustness and accuracy by choosing smoothing parameters per sample \citep{Alfarra20DataDependent}, however often at the cost of statistical soundness \citep{Sukenik21Intriguing}.

In this work, we build on ideas from \citet{mueller2021certify} to construct compositional architectures for probabilistic certification and propose corresponding statistically sound and efficient inference and certification procedures based on randomized smoothing \citep{CohenRK19}.
More concretely, we propose to use a smoothed selection-mechanism that adaptively chooses on a per-sample basis between a robustified smoothed classifier and a non-robust but highly accurate classifier.
We show that the synergy of RS with the proposed compositional architecture allows us to obtain significant robustness at almost no cost in terms of natural accuracy even on challenging datasets such as \IN while fully exposing this robustness-accuracy trade-off, even after training.

\textbf{Main Contributions} Our key contributions are:
\begin{itemize}[labelindent=1.9em,labelsep=0.25cm,leftmargin=*]
	\vspace{-2.5mm}
    \item{We are first to extend compositional architectures to the probabilistic certification setting, combining an arbitrary deep model with a smoothed classifier and selection-mechanism.}
    \item{We investigate two selection-mechanisms for choosing, at inference time and on a per-sample basis, between a robust and an accurate classifier and derive corresponding statistically sound prediction and certification algorithms.} %
	\item{We conduct an extensive empirical investigation of our compositional architectures on \IN and \cifar and find that they achieve significantly more attractive trade-offs between robustness and accuracy than any current method. On \IN, we, e.g., achieve $15.8\%$ more natural accuracy at the same ACR or $0.14$ more ACR at the same natural accuracy.}
\end{itemize}

\section{Background \& Related Work}
\label{sec:background}
\vspace{-1mm}
In this section, we review related work and relevant background.
\vspace{-1.5mm}
\paragraph{Adversarial Robustness \& Threat Model} 
Let $\vf \colon \R^d \mapsto \R^{\nclass}$ be a classifier computing an $\nclass$-dimensional logit vector, assigning a numerical score to each of the $\nclass$ classes, given a $d$-dimensional input.
Additionally, let $F(\vx) := \argmax_i f(\vx)_i$ with $F: \R^d \mapsto [1, \dots, \nclass]$ be the function that outputs the class with the largest score.
On a given input $\vx$ with label $y$, we say $F$ is (accurately) adversarially robust if it classifies all inputs in a $p$-norm ball $B_\delta^p(\vx)$ of radius $\delta$ around the sample $\vx$ correctly: $ F(\vx) = F(\vx') = y, \forall \vx' \in B_\delta^p(\vx)$.
We distinguish between empirical and certified robustness. Empirical robustness is computed by trying to find a counterexample $\vx' \in B_\delta^p(\vx)$ such that $F(\vx') \neq F(\vx)$; it constitutes an upper bound to the true robust accuracy. Certified robustness, in contrast, constitutes a sound lower bound. We further distinguish probabilistic and deterministic certification:
Deterministic methods compute the reachable set for given input specifications \citep{katz2017reluplex, GehrMDTCV18, RaghunathanSL18a,  ZhangWCHD18, singh2019abstract} to then reason about the output.
While providing state-of-the-art guarantees for $\ell_{\infty}$ specifications, these methods are computationally expensive and typically limited to small networks. %
Probabilistic methods \citep{LiCWC19, LecuyerAG0J19, CohenRK19} construct a robustified classifier and obtain probabilistic robustness guarantees by introducing noise into the classification process, allowing the certification of much larger models.
In this work, we focus on probabilistic certification and an $\ell_2$-norm based threat model. Extensions to other threat models are orthogonal to our approach.

\vspace{-1mm}
\paragraph{Randomized Smoothing}
Randomized Smoothing (RS) \citep{CohenRK19} is one of the most popular probabilistic certification methods. The key idea is to generate many randomly perturbed instances of the same sample and to then conduct majority voting over the predictions on these perturbed samples.
More concretely, Randomized Smoothing constructs the smoothed classifier $\bar{F} \colon \R^d \mapsto [1, \dots, \nclass]$ by conducting majority voting over a random noise term $\epsilon \sim \bc{N}(0, \sigma_{\epsilon}^2 \mI)$:
\begin{equation}
  \label{eq:g}
  \bar{F}(\vx) := \argmax_c \mathbb{E}_{\epsilon \sim \bc{N}
        (0, \sigma_{\epsilon}^2 \mI)}(F(\vx + \epsilon) = c).
\end{equation}
For this smoothed classifier $\bar{F}$, we obtain the following robustness guarantee:
\begin{restatable}{theorem}{rs}\label{thm:original} \textnormal{(\citet{CohenRK19})}\textbf{.}
    Let $c_A \in [1, \dots, \nclass]$, $\epsilon \sim \bc{N}(0, \sigma_{\epsilon}^2 \mI)$, and $\underline{p_A}, \overline{p_B} \in [0,1]$. If
    \begin{equation}
	\label{eq:smooth}
        \prob_{\epsilon}(F(\vx + \epsilon) = c_A)
        \geq
        \underline{p_A}
        \geq
        \overline{p_B}
        \geq
        \max_{c \neq c_A}\prob_{\epsilon}(F(\vx + \epsilon) = c),
    \end{equation}
    \vspace{-0.5mm}
    then $\bar{F}(\vx + \delta) = c_A$ for all $\delta$ satisfying $\|\delta\|_2 < R$
    with $R := \tfrac{\sigma_{\epsilon}}{2}(\Phi^{-1}(\underline{p_A}) - \Phi^{-1}(\overline{p_B}))$.
\vspace{-0.5mm}
\end{restatable}

Where $\Phi^{-1}$ is the inverse Gaussian CDF. The expectation and probabilities in \cref{eq:g,eq:smooth}, respectively, are computationally intractable. Hence, \citet{CohenRK19} propose to bound them using Monte Carlo sampling and the Clopper-Pearson lemma \citep{clopper34confidence}.
We denote obtaining a class $c_A$ and radius $R$ fulfilling \cref{thm:original} as \emph{certification} and just obtaining the class as \emph{prediction}. In practice, both are computed with confidence $1-\alpha$. When this fails, we abstain from making a classification, denoted as \abstain.
Performance is typically measured in certified accuracy at radius $r$ ($R \geq r$) and average certified radius over samples (ACR). We focus on their trade-off with natural accuracy (NAC) and provide detailed algorithms and descriptions in \cref{sec:appendix-rs}.

\vspace{-1mm}
\paragraph{Trade-Off}
For both empirical and certified methods, it has been shown that there is a trade-off between model accuracy and robustness \citep{zhang2019theoretically,XieTGWYL20,Raghunathan19AdvCanHurt,TsiprasSETM19}.
In the case of RS, the parameter $\sigma_\epsilon$ provides a natural way to trade-off certificate strength and natural accuracy \citep{CohenRK19, Mohapatra21HiddenCost}.%

\vspace{-1mm}
\paragraph{Compositional Architectures For Deterministic Certification (\ace)}
To enable efficient robustness-accuracy trade-offs for deterministic certification, \citet{mueller2021certify} introduced a compositional architecture.
The main idea of their \ace architecture is to use a selection model to certifiably predict certification-difficulty, and depending on this, either classify using a model with high certified accuracy,  $F_{\text{Certify}}: \R^d \mapsto [1, \dots, m]$, or a model with high natural accuracy, $F_{\text{Core}}: \R^d \mapsto [1, \dots, m]$.
Overall, the \ace architecture $F_{\ace}: \R^d \mapsto [1, \dots, m]$ is defined as
\begin{equation}
	\label{eq:ace}	
	F_{\ace}(\bm{x}) = F_{\text{Select}}(\bm{x}) \cdot F_{\text{Certify}}(\bm{x}) + (1-F_{\text{Select}}(\bm{x})) \cdot F_{\text{Core}}(\bm{x}).
\end{equation}
\citet{mueller2021certify} propose two instantiations for the selection-mechanism, $F_{\text{Select}}: \R^d \mapsto \{0,1\}$: a learned binary classifier and a mechanism selecting $F_{\text{Certify}}$ if and only if the entropy of its output is below a certain threshold.
In order to obtain a certificate, both $F_\text{Certify}$ and $F_{\text{Select}}$ must be certified.

\section{Robustness vs. Accuracy Trade-Off via Randomized Smoothing}
\label{sec:ace_smoothing}

\vspace{-0.5mm}
Here, we introduce \tool which instantiates \ace (\cref{eq:ace}) with Randomized Smoothing by replacing
$F_{\text{Select}}$  and $F_{\text{Certify}}$ with their smoothed counterparts $\bar{F}_{\text{Select}}$ and $\bar{F}_{\text{Certify}}$, respectively:
\begin{equation}
	\label{eq:aces}	
	F_{\tool}(\bm{x}) = \bar{F}_{\text{Select}}(\bm{x}) \cdot \bar{F}_{\text{Certify}}(\bm{x}) + (1-\bar{F}_{\text{Select}}(\bm{x})) \cdot F_{\text{Core}}(\bm{x}).
\end{equation}
Note that, due to the high cost of certification and inference of smoothed models, instantiating $F_{\text{Core}}$ with significantly larger models than $F_{\text{Certify}}$ and $F_{\text{Select}}$ comes at a negligible computational cost.

\begin{wrapfigure}[17]{r}{0.56\textwidth}
\vspace{-8.0mm}
\scalebox{0.92}{
	\begin{minipage}{1.05\linewidth}
	\input{algorithm-certify-ace}
	\end{minipage}
}
\end{wrapfigure}
\paragraph{Prediction \& Certification}
Just like other smoothed models (\cref{eq:g}), \tool (\cref{eq:aces}) can usually not be evaluated exactly in practice but has to be approximated via sampling and confidence bounds.
We thus propose \textsc{Certify} (shown in \cref{alg:certify-aces}) to soundly compute the output $F_\text{\tool}(\vx)$ and its robustness radius $R$.
Here, \textsc{SampleWNoise}($f, \vx, n, \sigma_{\epsilon}$) evaluates $n$ samples of $f(\vx + \epsilon)$ for $\epsilon \!\! \sim \!\! \bc{N}(0,\sigma_{\epsilon}\mI)$, and \textsc{LowerConfBnd}($m,n,c$) computes a lower bound to the success probability $p$ for obtaining $m$ successes in $n$ Bernoulli trials with confidence $c$.
Conceptually, we apply the \textsc{Certify} procedure introduced in \citet{CohenRK19} twice, once for
$\bar{F}_{\text{Select}}$ and once for $\bar{F}_{\text{Certify}}$.
If $\bar{F}_{\text{Select}}$ certifiably selects the certification model, we evaluate $\bar{F}_{\text{Certify}}$ and return its prediction $\hat{c}_A$ along with the minimum certified robustness radius of $\bar{F}_{\text{Select}}$ and $\bar{F}_{\text{Certify}}$.
If $\bar{F}_{\text{Select}}$ certifiably selects the core model, we directly return its classification $F_{\text{Core}}(\vx)$ and no certificate ($R=0$).
If $\bar{F}_{\text{Select}}$ does not certifiably select either model, we either return the class that the core and certification model agree on or abstain ($\abstain$).
A robustness radius $R$ obtained this way holds with confidence $1-\alpha$ (\cref{thm:aces_cert} in \cref{sec:appendix-prediction}). Note that individual tests need to be conducted with  $1-\tfrac{\alpha}{2}$ to account for multiple testing \citep{bonferroni1936teoria}.
Please see \cref{sec:appendix-prediction} for a further discussion and \textsc{Predict}, an algorithm computing $F_\text{\tool}(\vx)$ but not $R$ at a lower computational cost.

\vspace{-1.2mm}

\paragraph{Selection Model}
We can apply RS to any binary classifier $F_{\text{Select}}$ to obtain a smoothed selection model $\bar{F}_{\text{Select}}$.
Like \citet{mueller2021certify}, we consider two selection-mechanisms: i) a separate selection-network framing selection as binary classification and ii) a mechanism based on the entropy of the certification-network's logits $\vf_{\text{Certify}}(\vx)$ defined as $F_{\text{Select}}(\vx,\theta) := \mathds{1}_{\mathcal{H}(\softmax(\vf_{\text{Certify}}(\vx))) \leq \theta}$ where $\theta \in \mathbb{R}$ denotes the selection threshold.
While a separate selection-network performs much better in the deterministic setting \citep{mueller2021certify}, we find that in our setting the entropy-based mechanism is even more effective (see \cref{sec:appendix-selection-model}).
Thus, we focus our evaluation on an entropy-based selection-mechanism. %
Using such a selection-mechanism allows us to evaluate \tool for a large range of $\theta$, thus computing the full Pareto frontier (shown in \cref{fig:ace_smoothadv}), without reevaluating $\bar{F}_{\text{Certify}}$ and $F_{\text{Core}}$. This makes the evaluation of \tool highly computationally efficient.
We can even evaluate all component models separately and compute \tool certificates for arbitrary combinations retrospectively, allowing quick evaluations of new component models. 
\vspace{-0.5mm}

\section{Experimental Evaluation}
\label{sec:experimental-evaluation}
\vspace{-0.5mm}

\begin{wrapfigure}[9]{r}{0.39\textwidth}
	\centering
	\vspace{-12mm}
	\includegraphics[width=0.9\linewidth]{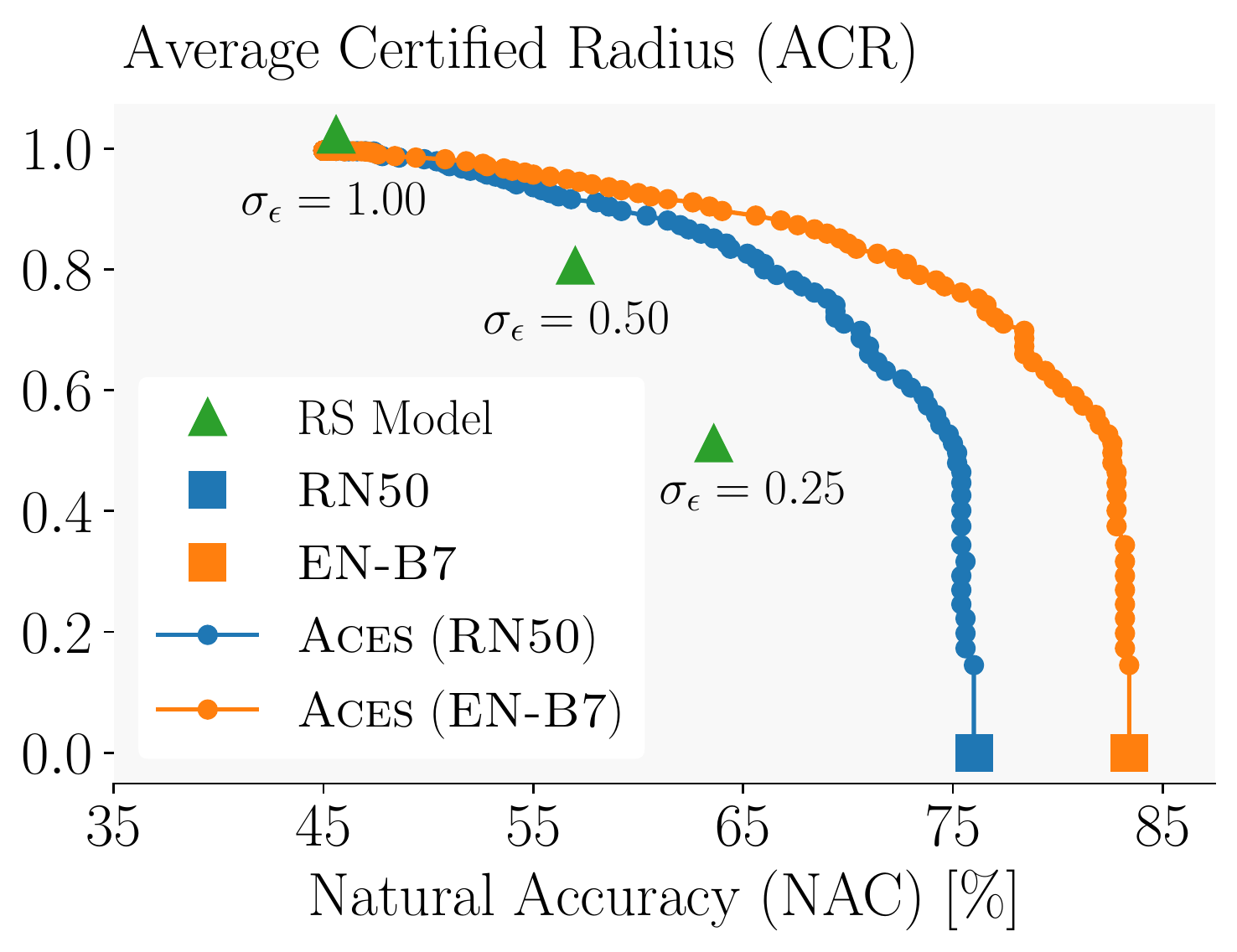}
	\vspace{-3mm}
	\caption{ACR over NAC on \IN.}
	\label{fig:ace_smoothadv}
	\vspace{-1.5mm}
\end{wrapfigure}

In this section, we evaluate \tool on the \IN and \cifar datasets and demonstrate that it yields much higher average certified radii (ACR) and certified accuracies at a wide range of natural accuracies (NAC) than current state-of-the-art methods.
Please see \cref{sec:appendix-experimental-details} for a detailed description of the experimental setup and \cref{sec:appendix-additional-experiments} for significantly extended results, including different training methods and noise levels $\sigma$, showing that the effects discussed here are consistent across a wide range of settings.

\begin{table}[t]
	\centering
	\small
	\centering
	\caption{Comparison of natural accuracy (NAC), average certified radius (ACR), and certified accuracy and selection rate at various radii on \IN with $\sigma_{\epsilon}=0.5$. We use a \consistency trained \RNM as certification-network and an \EN as core-network.}
	\label{tab:IN_main_paper}
	\vspace{-2mm}
	\scalebox{0.75}{
		\begin{threeparttable}
			\begin{tabular}{ccccccccccccccccc}
				\toprule
				\multirow{2.6}{*}{$\theta$} &\multirow{2.6}{*}{NAC} & \multirow{2.6}{*}{ACR} & \multicolumn{7}{c}{Certified Accuracy at Radius r}&  \multicolumn{7}{c}{Certified Selection Rate at Radius r}\\
				\cmidrule(lr){4-10} 				\cmidrule(lr){11-17}
				& & & 0.00 & 0.25 & 0.50 & 0.75 & 1.00 & 1.25 & 1.50 & 0.00 & 0.25 & 0.50 & 0.75 & 1.00 & 1.25 & 1.50\\
				\midrule
				0.0 & 83.4 & 0.000 & 83.4 & 0.0 & 0.0 & 0.0 & 0.0 & 0.0 & 0.0 & 0.0 & 0.0 & 0.0 & 0.0 & 0.0 & 0.0 & 0.0\\
				0.1 & 80.0 & 0.530 & 80.0 & 33.6 & 32.6 & 30.2 & 28.2 & 25.6 & 23.0 & 45.0 & 40.2 & 37.2 & 34.0 & 31.8 & 28.2 & 25.0\\
				0.2 & 75.4 & 0.682 & 75.0 & 43.6 & 41.2 & 38.2 & 35.8 & 33.4 & 30.0 & 63.8 & 58.6 & 55.6 & 50.6 & 47.8 & 45.2 & 40.8\\
				0.3 & 68.8 & 0.744 & 68.2 & 48.4 & 44.4 & 41.6 & 39.2 & 35.6 & 32.8 & 78.0 & 74.2 & 70.2 & 66.2 & 62.8 & 59.0 & 55.0 \\
				0.6 & 57.2 & 0.799 & 55.4 & 51.6 & 48.8 & 45.0 & 42.0 & 39.0 & 34.6 & 99.8 & 99.4 & 99.0 & 98.2 & 97.4 & 96.6 & 94.6\\
				1.0 & 57.2 & 0.800 & 55.4 & 51.6 & 48.8 & 45.2 & 42.2 & 39.0 & 34.6 & 100.0 & 100.0 & 100.0 & 100.0 & 100.0 & 100.0 & 100.0\\ 
				\bottomrule
			\end{tabular}
		\end{threeparttable}
	}
	\vspace{-4mm}
\end{table}

\paragraph{\tool on \IN}
\cref{fig:ace_smoothadv} compares the average certified radius (ACR) over natural accuracy (NAC) obtained on \IN by individual \RNM (green triangles) with those obtained by \tool (dots).
We use \RNM with $\sigma_{\epsilon}=1.0$ as certification-networks and either another \RNM (blue) or an \EN (orange) as the core-network (squares) for \tool. There, the horizontal gap between the individual RS models (triangles) and \tool (orange line) corresponds to the increase in natural accuracy at the same robustness, e.g., $15.8\%$ for $\sigma_{\epsilon}=0.5$.
We further observe that \tool already dominates the ACR of the individual models, especially at high natural accuracies, when using the small \RNM as core-network and even more so with the stronger \EN.%

\cref{tab:IN_main_paper} shows how the certified accuracy and selection rate (ratio of samples sent to the certification-network) change with the selection threshold $\theta$.
Increasing $\theta$ from $0.0$ to $0.1$ only reduces natural accuracy by $3.4\%$ while increasing ACR from $0.0$ to $0.530$ and certified accuracy at $r=1.0$ from $0.0\%$ to $28.2\%$.
Similarly, reducing $\theta$ from $1.0$ to $0.3$ loses very little ACR ($0.056$) and certified accuracy ($3.0\%$ at $r=1.0$) but yields a significant gain in natural accuracy ($11.6\%$).

\paragraph{\tool on \cifar}
\begin{wrapfigure}[13]{r}{0.463\textwidth}
	\centering
	\vspace{-5mm}
	\includegraphics[width=0.97\linewidth]{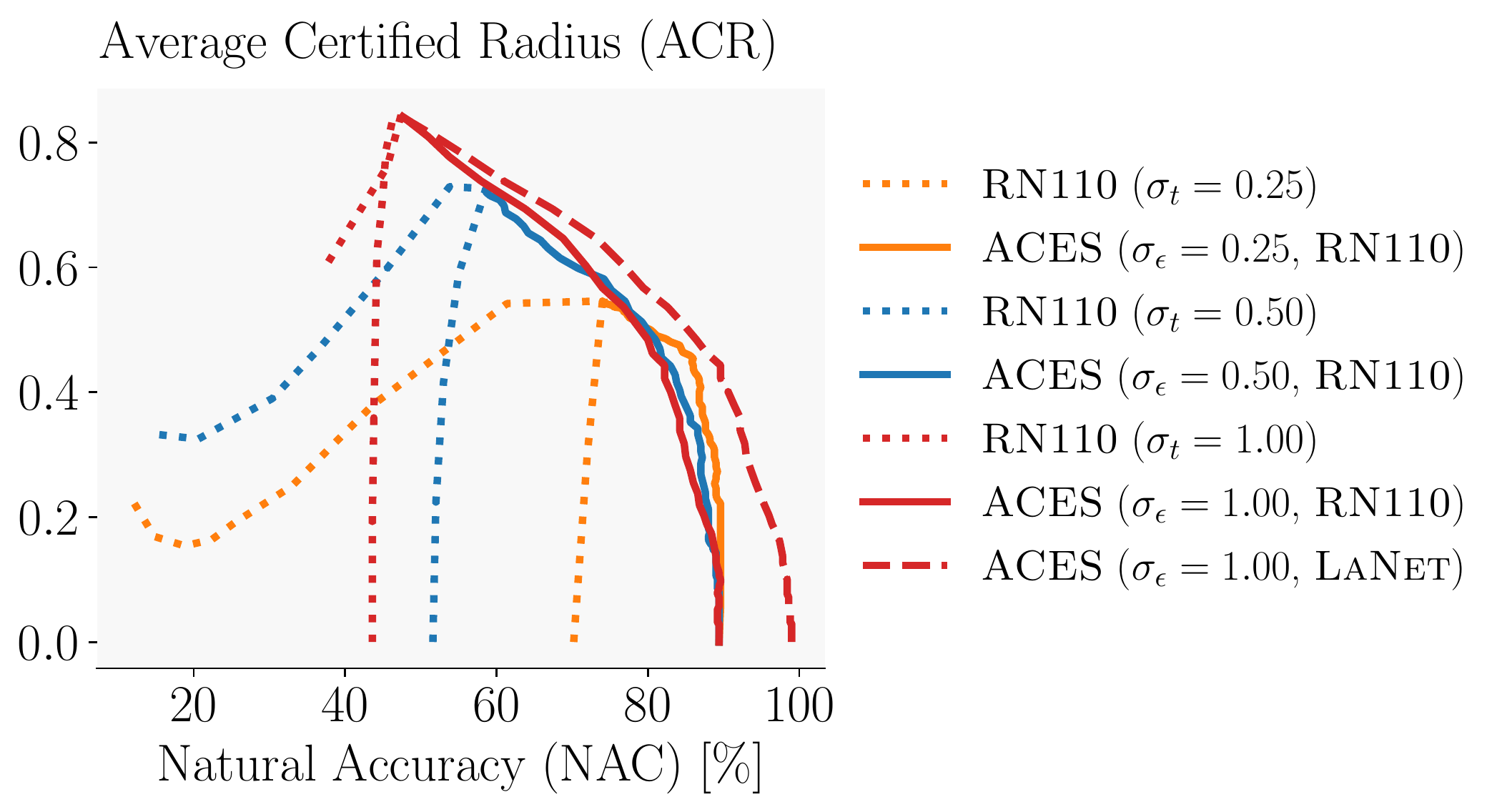}
	\vspace{-2mm}
	\caption{Comparison of ACR over natural accuracy of \tool with different noises $\sigma_{\epsilon}$ and selection thresholds $\theta$ (solid \& dashed lines), and individual \RNB evaluated with $\sigma_e \in [0.0, 1.5]$ and trained at $\sigma_t \in \{0.25, 0.5, 1.0\}$.}
	\label{fig:ace_smoothadv_cifar}
	\vspace{-1.5mm}
\end{wrapfigure}
\cref{fig:ace_smoothadv_cifar} compares \tool (solid \& dashed lines) against a baseline of varying the inference noise levels $\sigma_{\epsilon}$ (dotted lines) with respect to the robustness accuracy trade-offs obtained on \cifar.
Using only \RNB, \tool models (solid lines) dominate all individual models across training noise levels $\sigma_t \in \{0.25, 0.5, 1.0\}$ (orange, blue, red).
Individual models only reach comparable performance when evaluated at their training noise level.
However, covering the full Pareto frontier this way would require training a very large number of networks to match a single \tool model. %
Using a more precise \LANET as core-network for \tool (red dashed line) significantly widens this gap.

\paragraph{Selection-Mechanism}
\begin{wrapfigure}[8]{r}{0.31\textwidth}
	\centering
	\vspace{-5mm}
	\includegraphics[width=0.95\linewidth]{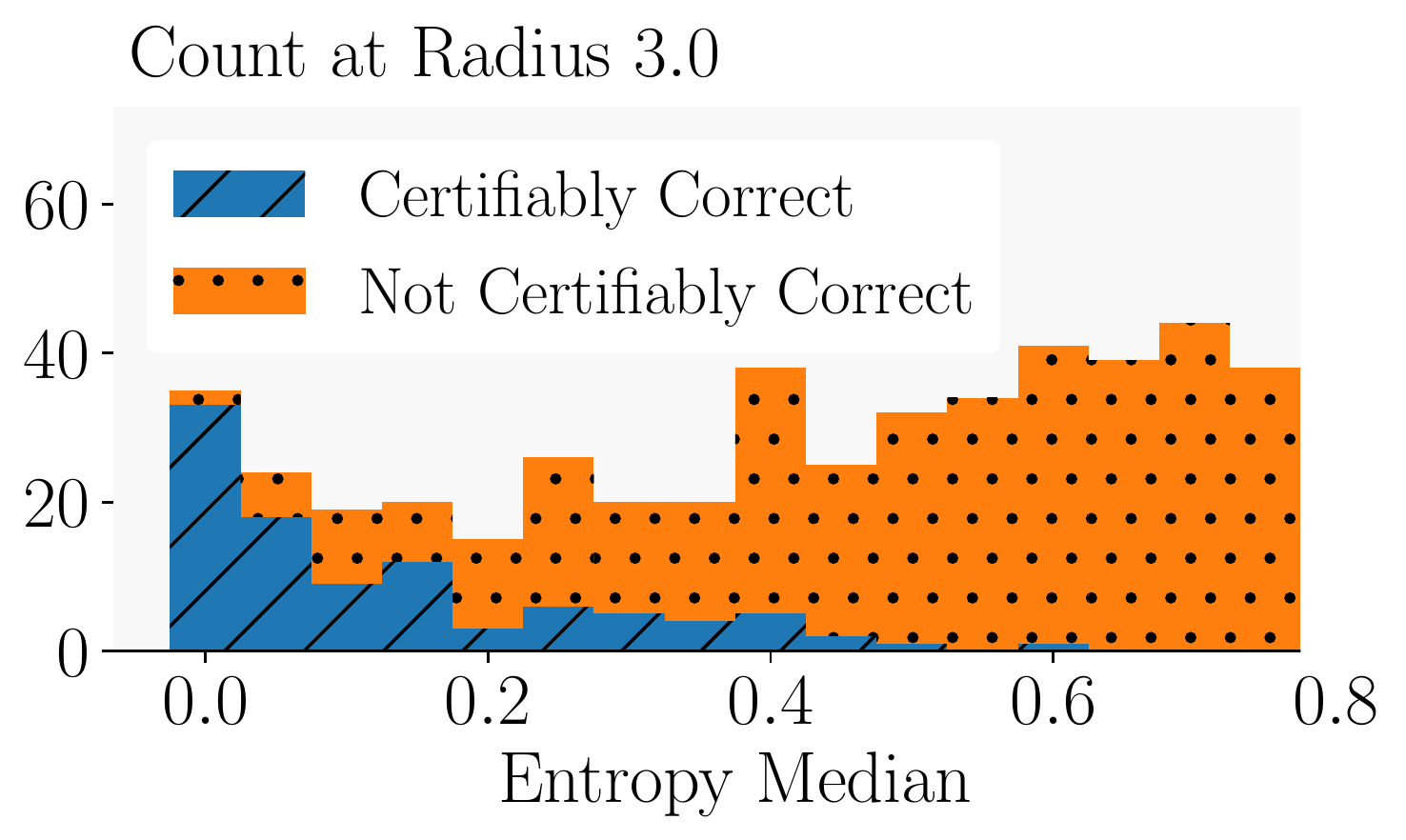}
	\vspace{-4mm}
	\caption{Certifiable correctness over median entropy.}
	\label{fig:entropies_robust}
\end{wrapfigure}
In \cref{fig:entropies_robust}, we visualize the distribution of samples that can (blue) and can not (orange) be certified correctly (at $r=3.0$) over the certification-network's median entropy (over perturbations).
Samples to the left of a chosen threshold are assigned to the certification-network and the rest to the core-network.
While separation is not perfect, we observe that there is a quick decline in the portion of certifiable samples as entropy increases, indicating that the selection-mechanism works well.

\section{Conclusion}
We extend compositional architectures to probabilistic robustness certification, achieving, for the first time, both high certifiable \emph{and} natural accuracies on the challenging \IN dataset. The key component of our \tool architecture is a certified, entropy-based selection-mechanism, choosing, on a per-sample basis, whether to use a smoothed model yielding guarantees or a more accurate standard model for inference. Our experiments show that \tool yields trade-offs between robustness and accuracy that are beyond the reach of current state-of-the-art approaches while being fully orthogonal to other improvements of Randomized Smoothing.

\message{^^JLASTBODYPAGE \thepage^^J}

\clearpage
\bibliography{references}
\bibliographystyle{plainnat}

\message{^^JLASTREFERENCESPAGE \thepage^^J}

\ifbool{includeappendix}{%
	\clearpage
	\appendix

\section{Randomized Smoothing}
\label{sec:appendix-rs}

\begin{figure}[h]
	\vspace{-4mm}
	\input{algorithm-certify}
	\vspace{-6mm}
	\input{algorithm-predict}
	\vspace{-4mm}
\end{figure}

In this section, we briefly explain the practical certification and inference algorithms \textsc{Certify} and \textsc{Predict}, respectively,  for a smoothed classifier
\begin{equation*}
\bar{F}(\vx) := \argmax_c \mathbb{E}_{\epsilon \sim \bc{N}
	(0, \sigma_{\epsilon}^2 \mI)}(F(\vx + \epsilon) = c)
\end{equation*}
as introduced by \cite{CohenRK19}. We first define some components of \cref{alg:predict,alg:certify-rs} below before we discuss them in more detail:

$\textsc{SampleWNoise}(F, x, n, \sigma_{\epsilon})$ first samples $n$ inputs $x_1, \dots, x_n$ as $x_i = x + \epsilon_i$ for $\epsilon_i \sim \mathcal{N}(0, \sigma_{\epsilon})$.
Then it counts how often $F$ predicts which class for these $x_1, \dots, x_n$ and returns the corresponding $\nclass$ dimensional array of counts.

$\textsc{LowerConfBnd}(k, n, 1 - \alpha)$ returns a lower bound on the unknown probability $p$ with confidence at least $1 - \alpha$ such that $k \sim \mathcal{B}(n, p)$ for the binomial distribution with parameters $n$ and $p$.

$\textsc{BinomPValue}(n_A,n,p)$ returns the probability of at least $n_A$ success in $n$ Bernoulli trials with success probability $p$.

\paragraph{Certification}
We first recall the robustness guarantee for a smoothed classifier (\cref{thm:original}):

\rs*

Unfortunately, computing the exact probabilities ${\prob_\epsilon(F(\vx + \epsilon)=c)}$ is generally intractable.
Thus, to allow practical application, \citet{CohenRK19} propose \textsc{Certify} (\cref{alg:certify-rs}) utilizing Monte Carlo sampling and confidence bounds:
First, we draw $n_0$ samples to determine the majority class $\hat{c}_A$. Then, we draw another $n$ samples to compute a lower bound $\underline{p_A}$ to the success probability, i.e., the probability of the underlying model to predict $\hat{c}_A$ for a perturbed sample, with confidence $1-\alpha$ via the Clopper-Pearson lemma \citep{clopper34confidence}.
If $\underline{p_A} > 0.5$, we set $\overline{p_{B}} = 1 - \underline{p_A}$ and obtain radius $R = \sigma_{\epsilon} \Phi^{-1}(\underline{p_A})$ via \cref{thm:original} with confidence $1-\alpha$, else we abstain (return \abstain). See \citet{CohenRK19} for a proof.

\paragraph{Prediction}
Computing a confidence bound to the success probability with \textsc{Certify} is computationally expensive as the number of samples $n$ is typically large. If we are only interested in computing the class predicted by the smoothed model, we can use the computationally much cheaper \textsc{Predicts} (\cref{alg:predict}) proposed by \citet{CohenRK19}.
Instead of sampling in two separate rounds, we only draw $n$ samples once and compute the two most frequently predicted classes $\hat{c}_A$ and $\hat{c}_B$ with frequencies $n_A$ and $n_B$, respectively. Subsequently, we test if the probability of obtaining $n_A$ success in $n_A+n_B$ fair Bernoulli trials is smaller than $\alpha$, and if so, have with confidence $1 - \alpha$ that the true prediction of the smoothed model is in fact $\hat{c}_A$. See \citet{CohenRK19} for a proof.

\paragraph{Training for Randomized Smoothing}
To obtain high certified radii via \textsc{Certify}, the base model $F$ has to be trained specifically to cope with the added noise terms $\epsilon$. To achieve this, several training methods have been introduced, which we quickly outline below.

\citet{CohenRK19} propose to use data augmentation with Gaussian noise during training. We refer to this as \cohen.
\citet{salman2019provably} suggest \smoothadv, combining adversarial training \citep{madry2017towards, KurakinGB17,rony2019decoupling} with data augmentation ideas from \cohen. %
While effective in improving accuracy, this training procedure comes with a very high computational cost.
\citet{zhai2020macer} propose \macer as a computationally cheaper alternative with a similar performance by adding a surrogate of the certification radius to the loss and thus more directly optimizing for large radii.
\citet{jeong2020consistency} build on this approach by replacing this term with a more easily optimizable one and proposing what we refer to as \consistency.

\section{Prediction \& Certification for \tool}
\label{sec:appendix-prediction}

\setcounter{algorithm}{0}
	\setlength{\itemindent}{0.0em}
	\algrenewcommand\algorithmicindent{0.5em}
	  \begin{algorithm}[H]
		  \caption{Certification for \tool}
		  \label{alg:certify-aces}
		  \begin{algorithmic}
			  \Function{Certify}{$\sigma_{\epsilon}, \vx, n_0, n, \alpha$}
			  \State $\texttt{counts}_S^0 \leftarrow \textsc{SampleWNoise}(F_{\text{Select}}, \vx, n_0, \sigma_{\epsilon})$

			  \State $\texttt{counts}_C^0 \leftarrow \textsc{SampleWNoise}(F_{\text{Certify}}, \vx, n_0, \sigma_{\epsilon})$
			  \State $\hat{s} \leftarrow 0 \text{\textbf{ if }} \texttt{counts}_S^0[0]>\texttt{counts}_S^0[1]  \text{\textbf{ else }} 1$
			  \State $\hat{c}_A \leftarrow$ top indices in $\texttt{counts}_C^0$
			  \State $\texttt{counts}_S \leftarrow \textsc{SampleWNoise}(F_{\text{Select}}, \vx, n, \sigma_{\epsilon})$
 			  \State $\texttt{counts}_C \leftarrow \textsc{SampleWNoise}(F_{\text{Certify}}, \vx, n, \sigma_{\epsilon})$
			  \State $\underline{p_S} \leftarrow \textsc{LowerConfBnd}$($\texttt{counts}_S[\hat{s}],n,1 - \tfrac{\alpha}{2}$)
			  \State $\underline{p_A} \leftarrow \textsc{LowerConfBnd}$($\texttt{counts}_C[\hat{c}_A],n,1 - \tfrac{\alpha}{2}$)
			  \State $\underline{p} \leftarrow \min(\underline{p_A}, \underline{p_S})$

			  \State \textbf{if} {$\hat{s} = 1 \land \underline{p} > \frac{1}{2}$} \textbf{return} $\hat{c}_A$ and $R:=\sigma_{\epsilon} \Phi^{-1}(\underline{p})$
			  \State \textbf{else if} $\hat{s} = 0 \land \underline{p_S} \geq \frac{1}{2}$ \textbf{return} $F_\text{Core}(\vx)$ and $R:=0$
			  \State \textbf{else if} {$\hat{c}_A = F_\text{Core}(\vx) \land \underline{p_A} \geq \frac{1}{2}$} \textbf{return} $\hat{c}_A$ and $R:=0$
			  \State \textbf{else} \textbf{return} \abstain and $R:=0$
			  \EndFunction
	  \end{algorithmic}
  \end{algorithm}

\setcounter{algorithm}{3}

In this section, we recall the certification approach (\cref{alg:certify-aces}) and introduce the prediction approach (\cref{alg:predict-ace}, below) in detail for \tool as discussed in \cref{sec:ace_smoothing}.

\paragraph{Certification}
For an arbitrary but fixed $\vx$ we let $c := F_\text{\tool}(\vx)$ denote the true output of \tool (\cref{eq:aces}) under exact evaluation of the expectations over perturbations (\cref{eq:g}) and let
\begin{equation*}
	R := \begin{cases}
		\min(R_\text{Select}, R_\text{Certify}) & \text{if } \bar{F}_{\text{Select}}(\vx) = 1\\
		0 & \text{otherwise}
	\end{cases},
\end{equation*}
where $R_\text{Select}, R_\text{Certify}$ denote the robustness radius according to \cref{thm:original}
for $\bar{F}_{\text{Select}}(\vx)$ and $\bar{F}_{\text{Certify}}(\vx)$, respectively. We now obtain the following guarantees for the outputs of our certification algorithm \textsc{Certify}:

\begin{theorem}\label{thm:aces_cert}
	Let $\hat{c}, \hat{R}$ denote the class and robustness radius returned by \textsc{Certify} (\cref{alg:certify-aces}) for input $\vx$.
	Then, this output $\hat{c}$, computed via sampling, is the true output $F_{\text{\tool}}(\vx + \delta) =: c = \hat{c} \quad \forall \delta \text{ with } \|\delta\|_2 \leq \hat{R}$ with confidence at least $1-\alpha$, if $\hat{c} \neq \abstain$.
\end{theorem}

\begin{proof}
	First, we note that, as \textsc{Certify} (\cref{alg:certify-rs}) in \citet{CohenRK19}, our \textsc{Certify} determines $\underline{p_A}$ and $\underline{p_S}$ with probability $1-\frac{\alpha}{2}$. Thus allowing us to upper bound $\overline{p_B} := 1-\underline{p_A}$ and giving us $\hat{R}_\text{Certify}$ via \cref{thm:original} and similarly $\hat{R}_\text{Select}$.

	Thus, if $\bar{F}_{\text{Select}}(\vx)$ returns $1$ (selecting the certification network) with confidence $1-\frac{\alpha}{2}$ and $\bar{F}_{\text{Certify}}(\vx)$ returns class $c$ with confidence $1-\frac{\alpha}{2}$, then we have via union bound with confidence $1-\alpha$ that $F_{\text{\tool}}(\vx)$ returns $\hat{c} = c$. 
	Further, the probabilities $\underline{p_A}$ and $\underline{p_S}$ induce the robustness radii $\hat{R}_\text{Select}$ and $\hat{R}_\text{Certify}$, respectively, via \cref{thm:original}. Thus we obtain the robustness radius $\hat{R} = \min(\hat{R}_\text{Select}, \hat{R}_\text{Certify})$ as their minimum.

	Should $\bar{F}_{\text{Select}}(\vx) = 0$ (selecting the core network), with probability $1-\frac{\alpha}{2}$ we return the deterministically computed $F_\text{Core} = \hat{c} = c$, trivially with confidence $1-\frac{\alpha}{2} \geq 1- \alpha$.
	As we only only claim robustness with $\hat{R} = 0$ in this case, the robustness statement is trivially fulfilled.

	In case we can not compute the decision of $\bar{F}_{\text{Select}}(\vx)$ with sufficient confidence, but $\bar{F}_{\text{Certify}}(\vx)$ and $F_{\text{Core}}(\vx)$ agree with high confidence, we return the consensus class. We again have trivially from the deterministic $F_{\text{Core}}$ and the prediction of $\bar{F}_{\text{Certify}}$ with confidence  $1-\frac{\alpha}{2}$ an overall confidence of $1-\frac{\alpha}{2} \geq 1- \alpha$ that indeed $\hat{c} = c$. Finally, in this case we again only claim $\hat{R} = 0$ which is trivially fulfilled.
\end{proof}

	  \begin{algorithm}[H]
		  \caption{Prediction for \tool}
		  \label{alg:predict-ace}
		  \begin{algorithmic}
			  \Function{Predict}{$\sigma_{\epsilon}, \vx, n, \alpha$}
			  \State $\texttt{counts}_S \leftarrow \textsc{SampleWNoise}(F_{\text{Select}}, \vx, n, \sigma_{\epsilon}$
			  \State $\texttt{counts}_C \leftarrow \textsc{SampleWNoise}(F_{\text{Certify}}, \vx, n, \sigma_{\epsilon})$
	  		  \State $n_0, n_1 \leftarrow \texttt{counts}_S[0], \texttt{counts}_S[1]$
			  \State $\hat{c}_A, \hat{c}_B \leftarrow$ top two indices in $\texttt{counts}_C$
  			  \State $n_A, n_B \leftarrow \texttt{counts}_C[\hat{c}_A], \texttt{counts}_C[\hat{c}_B]$
			  \State $\rho_A \leftarrow \textsc{BinomPValue}(n_A, n_A + n_B, 0.5)$
			  \State \textbf{if} {$ n_1 > n_0 \wedge \text{BinomPValue}(n_1, n, 0.5) \leq \tfrac{\alpha}{2} \land \rho_A \leq \tfrac{\alpha}{2}$} \textbf{return} $\hat{c}_A$
			  \State \textbf{else if } {$ n_0 > n_1 \lor \text{BinomPValue}(n_0, n, 0.5) \leq \tfrac{\alpha}{2}$} \textbf{return} $F_{\text{Core}}(\vx)$
			  \State \textbf{else if} {$\hat{c}_A = F_{\text{Core}}(\vx) \land \rho_A \leq \tfrac{\alpha}{2} $} \textbf{return} $\hat{c}_A$
			  \State \textbf{else} \textbf{return} \abstain
			  \EndFunction
	  \end{algorithmic}
  \end{algorithm}

\paragraph{Prediction}
Let us again consider the setting where for an arbitrary but fixed $\vx$ we $c := F_\text{\tool}(\vx)$ denotes the true output of \tool (\cref{eq:aces}) under exact evaluation of the expectations over perturbations (\cref{eq:g}). However, now we are only interested in the predicted class $\hat{c}$ and not the robustness radius. We thus introduce \textsc{Predict} (\cref{alg:predict-ace}), which is computationally much cheaper than \textsc{Certify} and for which we obtain the following guarantee:

\begin{theorem}
	Let $\hat{c}$ be the class returned by \textsc{Predict} (\cref{alg:predict-ace}) for input $\vx$. Then, this output computed via sampling is the true output $F_{\text{\tool}}(\vx) =: c = \hat{c}$ with confidence at least $1-\alpha$, if $\hat{c} \neq \abstain$ does not abstain.
\end{theorem}

\begin{proof}
	This proof follows analogously to that for \textsc{Certify} (\cref{thm:aces_cert}) from \citet{CohenRK19}.
\end{proof}

\section{Experimental Setup Details}
\label{sec:appendix-experimental-details}

In this section, we discuss experimental details.
We evaluated \tool on the \IN \citep{ImageNet} and the \cifar \citep{cifar} datasets.
For \IN, we combine \RNM \citep{He_2016_CVPR} selection- and certification-networks with \EN core-networks \citep{TanL19}. 
For \cifar, we use \RNB \citep{He_2016_CVPR} selection- and certification-networks, and \LANET \citep{Wang21LaNet} core-networks.
We implement training and inference in PyTorch \citep{PaszkeGMLBCKLGA19} and conduct all of our experiments on single GeForce RTX 2080 Ti.

As core-networks, we use pre-trained \EN\footnote{https://github.com/lukemelas/EfficientNet-PyTorch/tree/master/examples/imagenet} and \LANET \citep{Wang21LaNet} for \IN and \cifar, respectively.
As certification-networks, we use pre-trained \RNM and \RNB from \citet{CohenRK19} (\cohen ), \citet{salman2019provably} (\smoothadv), and \citet{zhai2020macer} (\macer).
Additionally, we train smoothed models with \consistency \citep{jeong2020consistency} using the parameters reported to yield the largest ACR, except on \IN with $\sigma_{\epsilon}=0.25$ where we use $\eta=0.5$ and $\lambda=5$ (there, no parameters were reported).

We follow previous work \citep{CohenRK19,salman2019provably} and evaluate every 20$^{th}$ image of the \cifar test set and every 100$^{th}$ of the \IN test set \citep{CohenRK19,jeong2020consistency}, yielding 500 test samples for each.
For both, we use $n_0=100$ and $n=100'000$ for certification, and $n=10'000$ for prediction (to report natural accuracy).
To obtain an overall confidence of $\alpha=0.001$ via Bonferroni correction \citep{bonferroni1936teoria}, we use $\alpha'=0.0005$ to certify the selection and the certification model.
To compute the entropy, we use the logarithm with basis $\nclass$ (number of classes), %
such that the resulting entropies are always in $[0, 1]$.
Certifying and predicting an \tool model on the 500 test samples we consider takes approximately $23.8$ hours on \IN, and $10.8$ hours on \cifar overall, using one RTX 2080 Ti.
This includes computations for a wide range ($>100$) values for the selection threshold $\theta$.

\section{Additional Experiments}
\label{sec:appendix-additional-experiments}

In this section, we provide a significantly extended evaluation focusing on the following aspects:

In \cref{sec:appendix-additional-imagenet-experiments,sec:appendix-additional-cifar-experiments}, we evaluate \tool for different training methods and a range of noise levels $\sigma$ on \IN and \cifar, respectively.

In \cref{sec:appendix-selection-ablation}, we provide an in-depth analysis of the selection-mechanism, considering different measures of selection performance and both entropy-based selection and a separate selection-network.

In \cref{sec:appendix-baselines}, we discuss the robustness-accuracy trade-offs obtained by varying the noise level $\sigma_{\epsilon}$ used at inference.

\subsection{Additional Results on \IN}
\label{sec:appendix-additional-imagenet-experiments}

\begin{table}[p]
	\centering
	\small
	\centering
	\caption{Natural accuracy (NAC), average certified radius (ACR) and certified accuracy at different radii on \IN with $\sigma_{t}=\sigma_{\epsilon}=0.25$ for a range of threshold parameters $\theta$ and an \tool model with entropy selection, a \RNM certification-network and an \EN core-network.} 
	\label{tab:IN_main_025}
	\scalebox{0.9}{
	\begin{threeparttable}
\begin{tabular}{ccccccccccccc}
        \toprule
        \multirow{2.6}{*}{Training} & \multirow{2.6}{*}{$\theta$} &\multirow{2.6}{*}{NAC} & \multirow{2.6}{*}{ACR} & \multicolumn{9}{c}{Certified Accuracy at Radius r}\\
		\cmidrule(lr){5-13}
		& & & & 0.0 & 0.25 & 0.5 & 0.75 & 1.0 & 1.25 & 1.5 & 1.75 & 2.0\\
        \midrule
        \multirow{12}{*}{\cohen}
        & 0.00 & 83.4 & 0.000 & 83.4 & 0.0 & 0.0 & 0.0 & 0.0 & 0.0 & 0.0 & 0.0 & 0.0\\
        & 0.10 & 82.2 & 0.273 & 82.2 & 35.4 & 27.4 & 21.4 & 0.0 & 0.0 & 0.0 & 0.0 & 0.0\\
        & 0.20 & 80.0 & 0.382 & 80.0 & 47.6 & 40.2 & 30.4 & 0.0 & 0.0 & 0.0 & 0.0 & 0.0\\
        & 0.30 & 78.6 & 0.431 & 78.6 & 54.6 & 45.6 & 35.2 & 0.0 & 0.0 & 0.0 & 0.0 & 0.0\\
        & 0.40 & 75.2 & 0.454 & 74.6 & 56.6 & 47.4 & 36.6 & 0.0 & 0.0 & 0.0 & 0.0 & 0.0\\
        & 0.50 & 72.8 & 0.464 & 71.4 & 58.0 & 48.4 & 37.4 & 0.0 & 0.0 & 0.0 & 0.0 & 0.0\\
\bottomrule
        \end{tabular}
\end{threeparttable}
}
\vspace{-4mm}
\end{table}

\begin{table}[p]
	\centering
	\small
	\centering
	\caption{Natural accuracy (NAC), average certified radius (ACR) and certified selection rate (portion of samples selected for the certification-network) at different radii on \IN with $\sigma_{t}=\sigma_{\epsilon}=0.25$ for a range of threshold parameters $\theta$ and an \tool model with entropy selection, a \RNM certification-network and an \EN core-network.} 
	\label{tab:IN_selection_025}
	\scalebox{0.9}{
	\begin{threeparttable}
    \begin{tabular}{ccccccccccccc}
        \toprule
        \multirow{2.6}{*}{Training} & \multirow{2.6}{*}{$\theta$} &\multirow{2.6}{*}{NAC} & \multirow{2.6}{*}{ACR} & \multicolumn{9}{c}{Certified Selection Rate at Radius r}\\
		\cmidrule(lr){5-13}
		& & & & 0.0 & 0.25 & 0.5 & 0.75 & 1.0 & 1.25 & 1.5 & 1.75 & 2.0\\
        \midrule
        \multirow{12}{*}{\cohen}
        & 0.00 & 83.4 & 0.000 & 0.0 & 0.0 & 0.0 & 0.0 & 0.0 & 0.0 & 0.0 & 0.0 & 0.0\\
        & 0.10 & 82.2 & 0.273 & 47.0 & 37.8 & 29.0 & 22.2 & 0.0 & 0.0 & 0.0 & 0.0 & 0.0\\
        & 0.20 & 80.0 & 0.382 & 66.0 & 57.0 & 48.8 & 39.8 & 0.0 & 0.0 & 0.0 & 0.0 & 0.0\\
        & 0.30 & 78.6 & 0.431 & 76.6 & 70.4 & 61.4 & 53.0 & 0.0 & 0.0 & 0.0 & 0.0 & 0.0\\
        & 0.40 & 75.2 & 0.454 & 86.2 & 80.0 & 72.4 & 64.6 & 0.0 & 0.0 & 0.0 & 0.0 & 0.0\\
        & 0.50 & 72.8 & 0.464 & 92.4 & 87.4 & 81.8 & 75.4 & 0.0 & 0.0 & 0.0 & 0.0 & 0.0\\
\bottomrule
        \end{tabular}
\end{threeparttable}
}
\vspace{-4mm}
\end{table}

\begin{table}[p]
	\centering
	\small
	\centering
	\caption{Natural accuracy (NAC), average certified radius (ACR) and certified accuracy at different radii on \IN with $\sigma_{t}=\sigma_{\epsilon}=0.50$ for a range of threshold parameters $\theta$ and an \tool model with entropy selection, a \RNM certification-network and an \EN core-network.}  
	\label{tab:IN_main_050}
	\scalebox{0.9}{
	\begin{threeparttable}
    \begin{tabular}{ccccccccccccc}
        \toprule
        \multirow{2.6}{*}{Training} & \multirow{2.6}{*}{$\theta$} &\multirow{2.6}{*}{NAC} & \multirow{2.6}{*}{ACR} & \multicolumn{9}{c}{Certified Accuracy at Radius r}\\
		\cmidrule(lr){5-13}
		& & & & 0.0 & 0.25 & 0.5 & 0.75 & 1.0 & 1.25 & 1.5 & 1.75 & 2.0\\
        \midrule
        \multirow{12}{*}{\cohen}
        & 0.00 & 83.4 & 0.000 & 83.4 & 0.0 & 0.0 & 0.0 & 0.0 & 0.0 & 0.0 & 0.0 & 0.0\\
        & 0.10 & 82.4 & 0.380 & 82.4 & 29.2 & 25.6 & 23.2 & 20.0 & 16.4 & 13.4 & 9.8 & 0.0\\
        & 0.20 & 78.6 & 0.536 & 78.4 & 38.8 & 33.4 & 30.6 & 28.6 & 24.4 & 21.2 & 16.4 & 0.0\\
        & 0.30 & 74.2 & 0.619 & 73.4 & 44.0 & 40.2 & 35.2 & 31.4 & 29.0 & 24.6 & 19.6 & 0.0\\
        & 0.40 & 70.8 & 0.665 & 69.4 & 47.8 & 43.0 & 38.2 & 33.2 & 30.6 & 26.2 & 20.2 & 0.0\\
        & 0.50 & 65.8 & 0.693 & 64.4 & 50.2 & 44.4 & 40.6 & 35.4 & 31.8 & 27.0 & 20.8 & 0.0\\
\bottomrule
        \end{tabular}
\end{threeparttable}
}
\vspace{-4mm}
\end{table}

\begin{table}[p]
	\centering
	\small
	\centering
	\caption{Natural accuracy (NAC), average certified radius (ACR) and certified selection rate (portion of samples selected for the certification-network) at different radii on \IN with $\sigma_{t}=\sigma_{\epsilon}=0.50$ for a range of threshold parameters $\theta$ and an \tool model with entropy selection, a \RNM certification-network and an \EN core-network.} 
	\label{tab:IN_selection_050}
	\scalebox{0.9}{
	\begin{threeparttable}
    \begin{tabular}{ccccccccccccc}
        \toprule
        \multirow{2.6}{*}{Training} & \multirow{2.6}{*}{$\theta$} &\multirow{2.6}{*}{NAC} & \multirow{2.6}{*}{ACR} & \multicolumn{9}{c}{Certified Selection Rate at Radius r}\\
		\cmidrule(lr){5-13}
		& & & & 0.0 & 0.25 & 0.5 & 0.75 & 1.0 & 1.25 & 1.5 & 1.75 & 2.0\\
        \midrule
        \multirow{12}{*}{\cohen}
        & 0.00 & 83.4 & 0.000 & 0.0 & 0.0 & 0.0 & 0.0 & 0.0 & 0.0 & 0.0 & 0.0 & 0.0\\
        & 0.10 & 82.4 & 0.380 & 35.0 & 31.6 & 27.6 & 24.2 & 20.6 & 16.8 & 13.8 & 10.2 & 0.0\\
        & 0.20 & 78.6 & 0.536 & 53.2 & 47.8 & 41.6 & 37.2 & 35.6 & 30.8 & 27.2 & 23.4 & 0.0\\
        & 0.30 & 74.2 & 0.619 & 68.2 & 61.8 & 57.2 & 49.6 & 45.2 & 41.2 & 37.6 & 32.2 & 0.0\\
        & 0.40 & 70.8 & 0.665 & 78.2 & 73.0 & 69.4 & 62.8 & 58.2 & 53.2 & 47.0 & 40.6 & 0.0\\
        & 0.50 & 65.8 & 0.693 & 88.4 & 83.8 & 79.8 & 74.4 & 71.0 & 65.0 & 59.2 & 51.8 & 0.0\\
\bottomrule
        \end{tabular}
\end{threeparttable}
}
\vspace{-4mm}
\end{table}

\begin{table}[p]
	\centering
	\small
	\centering
	\caption{Natural accuracy (NAC), average certified radius (ACR) and certified accuracy at different radii on \IN with $\sigma_{t}=\sigma_{\epsilon}=1.00$ for a range of threshold parameters $\theta$ and an \tool model with entropy selection, a \RNM certification-network and an \EN core-network.} 
	\label{tab:IN_main_100}
	\scalebox{0.9}{
	\begin{threeparttable}
    \begin{tabular}{ccccccccccccc}
        \toprule
        \multirow{2.6}{*}{Training} & \multirow{2.6}{*}{$\theta$} &\multirow{2.6}{*}{NAC} & \multirow{2.6}{*}{ACR} & \multicolumn{9}{c}{Certified Accuracy at Radius r}\\
		\cmidrule(lr){5-13}
		& & & & 0.0 & 0.5 & 1.0 & 1.5 & 2.0 & 2.5 & 3.0 & 3.5 & 4.0\\
        \midrule
        \multirow{12}{*}{\cohen}
        & 0.00 & 83.4 & 0.000 & 83.4 & 0.0 & 0.0 & 0.0 & 0.0 & 0.0 & 0.0 & 0.0 & 0.0\\
        & 0.10 & 83.0 & 0.322 & 83.0 & 15.0 & 11.8 & 9.4 & 7.6 & 5.2 & 4.2 & 3.2 & 0.0\\
        & 0.20 & 80.6 & 0.513 & 80.6 & 22.8 & 18.2 & 15.2 & 11.8 & 9.4 & 7.4 & 4.6 & 0.0\\
        & 0.30 & 75.8 & 0.650 & 75.2 & 28.6 & 23.6 & 19.4 & 14.4 & 11.2 & 9.6 & 6.2 & 0.0\\
        & 0.40 & 71.0 & 0.741 & 70.4 & 32.4 & 27.6 & 22.2 & 16.8 & 12.6 & 10.4 & 7.4 & 0.0\\
        & 0.50 & 64.2 & 0.801 & 62.4 & 34.8 & 29.6 & 24.4 & 18.6 & 13.8 & 11.6 & 8.2 & 0.0\\
\bottomrule
        \end{tabular}
\end{threeparttable}
}
\vspace{-4mm}
\end{table}

\begin{table}[p]
	\centering
	\small
	\centering
	\caption{Natural accuracy (NAC), average certified radius (ACR) and certified selection rate (portion of samples selected for the certification-network) at different radii on \IN with $\sigma_{t}=\sigma_{\epsilon}=1.00$ for a range of threshold parameters $\theta$ and an \tool model with entropy selection, a \RNM certification-network and an \EN core-network.} 
	\label{tab:IN_selection_100}
	\scalebox{0.9}{
	\begin{threeparttable}
    \begin{tabular}{ccccccccccccc}
        \toprule
        \multirow{2.6}{*}{Training} & \multirow{2.6}{*}{$\theta$} &\multirow{2.6}{*}{NAC} & \multirow{2.6}{*}{ACR} & \multicolumn{9}{c}{Certified Selection Rate at Radius r}\\
		\cmidrule(lr){5-13}
		& & & & 0.0 & 0.5 & 1.0 & 1.5 & 2.0 & 2.5 & 3.0 & 3.5 & 4.0\\
        \midrule
        \multirow{12}{*}{\cohen}
        & 0.00 & 83.4 & 0.000 & 0.0 & 0.0 & 0.0 & 0.0 & 0.0 & 0.0 & 0.0 & 0.0 & 0.0\\
        & 0.10 & 83.0 & 0.322 & 19.8 & 16.4 & 12.8 & 10.2 & 8.2 & 5.6 & 4.4 & 3.4 & 0.0\\
        & 0.20 & 80.6 & 0.513 & 34.2 & 28.4 & 22.6 & 19.2 & 16.0 & 13.4 & 10.0 & 7.0 & 0.0\\
        & 0.30 & 75.8 & 0.650 & 48.0 & 41.0 & 34.4 & 29.0 & 23.8 & 19.4 & 15.4 & 11.2 & 0.0\\
        & 0.40 & 71.0 & 0.741 & 60.2 & 53.0 & 47.0 & 40.8 & 34.6 & 28.6 & 23.6 & 16.8 & 0.0\\
        & 0.50 & 64.2 & 0.801 & 73.8 & 65.4 & 57.8 & 52.8 & 47.8 & 40.4 & 33.8 & 24.6 & 0.0\\
\bottomrule
        \end{tabular}
\end{threeparttable}
}
\vspace{-4mm}
\end{table}

\begin{figure}[t]
	\centering
	\scalebox{0.97}{
	\makebox[\textwidth][c]{
        \begin{subfigure}[t]{0.32\textwidth}
			\centering
			\includegraphics[width=\textwidth]{./figures/main-plots/acr_imagenet_main-eps-converted-to.pdf}
		\end{subfigure}
		\hfill
		\begin{subfigure}[t]{0.32\textwidth}
			\centering
			\includegraphics[width=\textwidth]{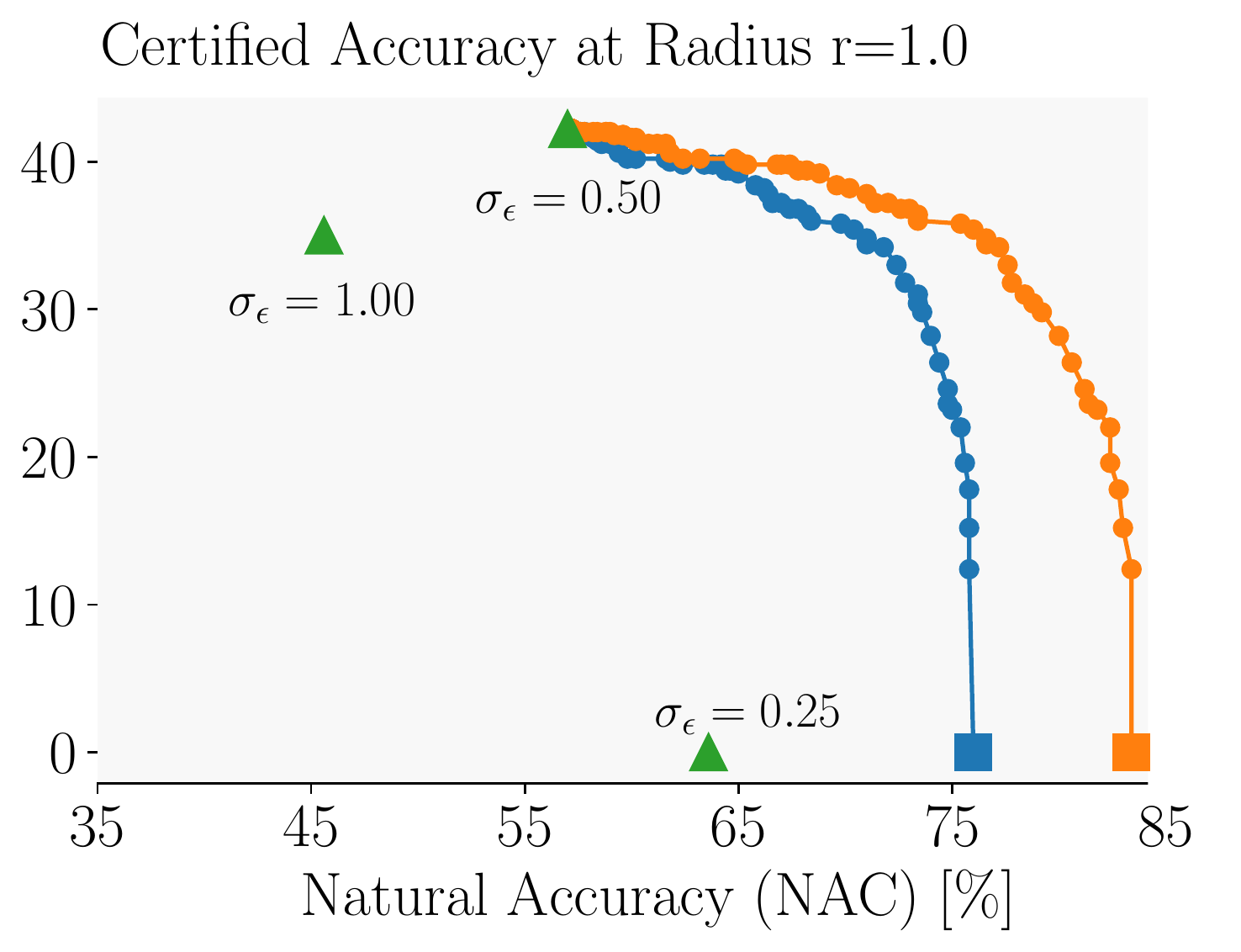}
		\end{subfigure}
		\hfill
		\begin{subfigure}[t]{0.32\textwidth}
			\centering
			\includegraphics[width=\textwidth]{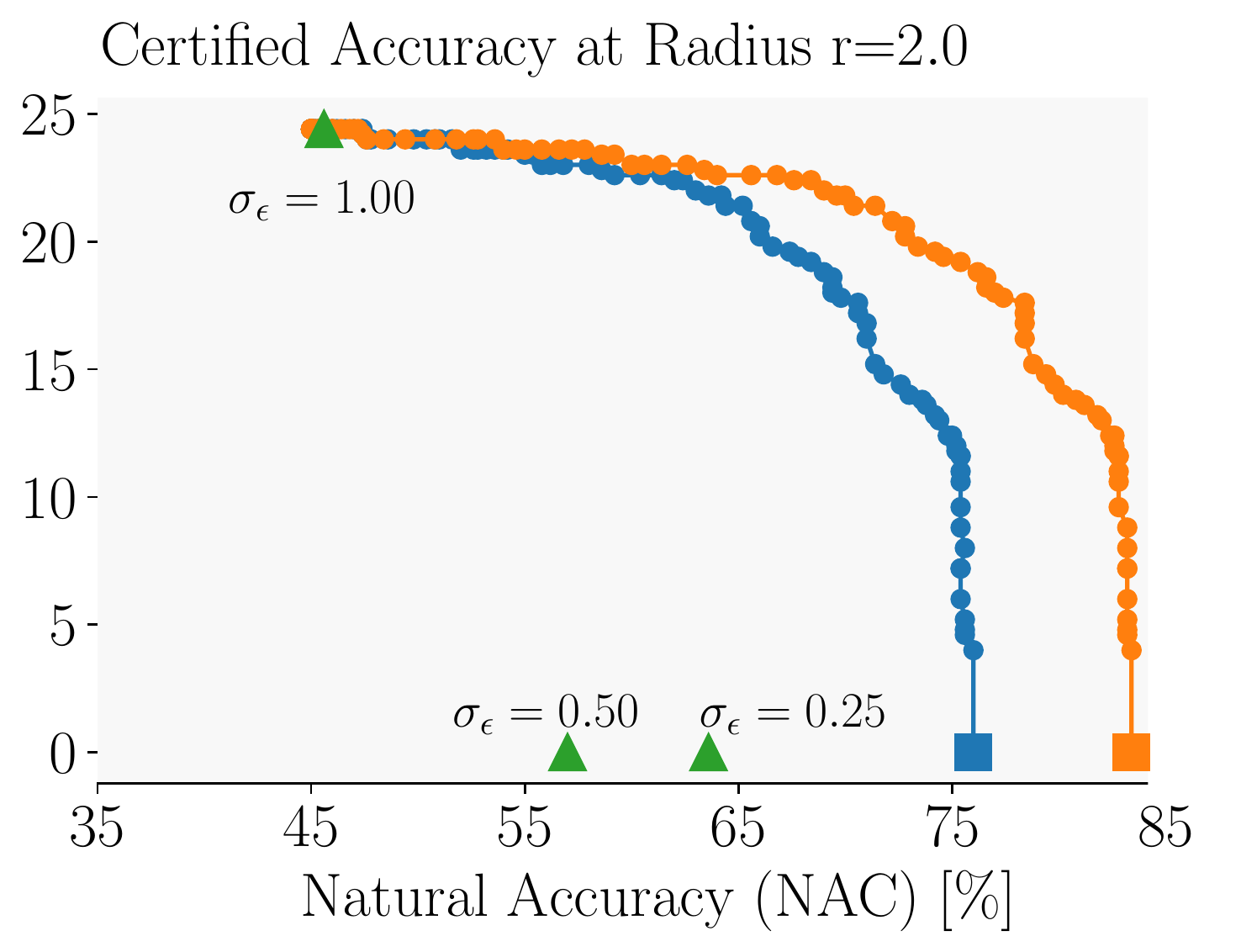}
		\end{subfigure}
	}
	}
	\vspace{-3mm}
	\caption{Comparison of \tool (blue and orange dots) and individual smoothed models (green triangles) on \IN with \consistency trained models with respect to average certified radius (left), certified accuracy at $r=1.0$ (middle), and certified accuracy at $r=2.0$ (right) over natural accuracy. We use \RNM for individual networks and as certification-networks for all \tool models. We consider \tool models with \RNM and \EN core-networks.}
	\label{fig:ace_imagenet_ca_appendix}
\end{figure}

In this section, we evaluate \tool on \IN for a wide range of training methods (\cohen, \smoothadv, and \consistency) and noise levels $\sigma \in \{0.25, 0.50, 1.00\}$.
In particular, we provide detailed results on the certified accuracies obtained by \tool in \cref{tab:IN_main_025} and the corresponding certified selection rates in \cref{tab:IN_selection_025} for $\sigma_{t} = \sigma_{\epsilon} = 0.25$.
Similarly, \cref{tab:IN_main_050,tab:IN_selection_050} and \cref{tab:IN_main_100,tab:IN_selection_100} contain results for $\sigma_{\epsilon} = 0.5$ and $\sigma_{\epsilon}=1.0$, respectively.

In \cref{fig:ace_imagenet_ca_appendix}, we visualize the trade-off between natural and certified accuracy at fixed radii for \tool (blue and orange dots) and individual smoothed models (green triangles).
We observe that \tool achieves significant certified accuracies at natural accuracies not achievable at all by conventional smoothed models.

For example, the highest natural accuracy ($63.6\%$) obtained by one of the \consistency smoothed models requires $\sigma_{\epsilon}=0.25$, leading to a certified accuracy of $0.0\%$ at $l_2$ radius $2.0$. \tool, in contrast, can use a certification-network with $\sigma_{\epsilon}=1.0$ to, e.g., obtain a similar natural accuracy of $66.8\%$ and a much higher certified accuracy of $22.6\%$.

\subsection{Additional Results on \cifar}
\label{sec:appendix-additional-cifar-experiments}

\begin{table}[p]
	\centering
	\small
	\centering
    \caption{Natural accuracy (NAC), average certified radius (ACR) and certified accuracy at different radii on \cifar with $\sigma_{t}=\sigma_{\epsilon}=0.25$ for a range of threshold parameters $\theta$ and an \tool model with entropy selection, a \RNB certification-network and an \LANET core-network.}  
	\label{tab:cifar10_main_025}
	\scalebox{0.9}{
	\begin{threeparttable}
\begin{tabular}{ccccccccccccc}
        \toprule
        \multirow{2.6}{*}{Training} & \multirow{2.6}{*}{$\theta$} &\multirow{2.6}{*}{NAC} & \multirow{2.6}{*}{ACR} & \multicolumn{9}{c}{Certified Accuracy at Radius r}\\
		\cmidrule(lr){5-13}
		& & & & 0.0 & 0.25 & 0.5 & 0.75 & 1.0 & 1.25 & 1.5 & 1.75 & 2.00\\
        \midrule
        \multirow{12}{*}{\cohen}   
        & 0.00 & 99.0 & 0.000 & 99.0 & 0.0 & 0.0 & 0.0 & 0.0 & 0.0 & 0.0 & 0.0 & 0.0\\
        & 0.10 & 98.0 & 0.189 & 98.0 & 28.8 & 18.4 & 9.6 & 0.0 & 0.0 & 0.0 & 0.0 & 0.0\\
        & 0.20 & 96.4 & 0.247 & 96.6 & 35.8 & 24.8 & 14.0 & 0.0 & 0.0 & 0.0 & 0.0 & 0.0\\
        & 0.30 & 94.6 & 0.303 & 94.8 & 41.6 & 29.6 & 19.6 & 0.0 & 0.0 & 0.0 & 0.0 & 0.0\\
        & 0.40 & 90.4 & 0.358 & 90.6 & 49.8 & 35.8 & 22.6 & 0.0 & 0.0 & 0.0 & 0.0 & 0.0\\
        & 0.50 & 85.4 & 0.397 & 85.2 & 56.0 & 40.4 & 24.2 & 0.0 & 0.0 & 0.0 & 0.0 & 0.0\\
\bottomrule
        \end{tabular}
\end{threeparttable}
}
\vspace{-4mm}
\end{table}

\begin{table}[p]
	\centering
	\small
	\centering
    \caption{Natural accuracy (NAC), average certified radius (ACR) and certified selection rate (portion of samples selected for the certification-network) at different radii on \cifar with $\sigma_{t}=\sigma_{\epsilon}=0.25$ for a range of threshold parameters $\theta$ and an \tool model with entropy selection, a \RNB certification-network and an \LANET core-network.} 
	\label{tab:cifar10_selection_025}
	\scalebox{0.9}{
	\begin{threeparttable}
    \begin{tabular}{ccccccccccccc}
        \toprule
        \multirow{2.6}{*}{Training} & \multirow{2.6}{*}{$\theta$} &\multirow{2.6}{*}{NAC} & \multirow{2.6}{*}{ACR} & \multicolumn{9}{c}{Certified Selection Rate at Radius r}\\
		\cmidrule(lr){5-13}
		& & & & 0.0 & 0.25 & 0.5 & 0.75 & 1.0 & 1.25 & 1.5 & 1.75 & 2.00\\
        \midrule
        \multirow{12}{*}{\cohen}
        & 0.00 & 99.0 & 0.000 & 0.0 & 0.0 & 0.0 & 0.0 & 0.0 & 0.0 & 0.0 & 0.0 & 0.0\\
        & 0.10 & 98.0 & 0.189 & 40.6 & 29.2 & 18.4 & 9.6 & 0.0 & 0.0 & 0.0 & 0.0 & 0.0\\
        & 0.20 & 96.4 & 0.247 & 51.0 & 36.4 & 25.2 & 14.0 & 0.0 & 0.0 & 0.0 & 0.0 & 0.0\\
        & 0.30 & 94.6 & 0.303 & 61.2 & 44.2 & 30.0 & 20.6 & 0.0 & 0.0 & 0.0 & 0.0 & 0.0\\
        & 0.40 & 90.4 & 0.358 & 76.0 & 56.8 & 40.6 & 28.6 & 0.0 & 0.0 & 0.0 & 0.0 & 0.0\\
        & 0.50 & 85.4 & 0.397 & 87.0 & 70.0 & 52.8 & 36.6 & 0.0 & 0.0 & 0.0 & 0.0 & 0.0\\
\bottomrule
        \end{tabular}
\end{threeparttable}
}
\vspace{-4mm}
\end{table}

\begin{table}[p]
	\centering
	\small
	\centering
    \caption{Natural accuracy (NAC), average certified radius (ACR) and certified accuracy at different radii on \cifar with $\sigma_{t}=\sigma_{\epsilon}=0.50$ for a range of threshold parameters $\theta$ and an \tool model with entropy selection, a \RNB certification-network and an \LANET core-network.} 
	\label{tab:cifar10_main_050}
	\scalebox{0.9}{
	\begin{threeparttable}
    \begin{tabular}{ccccccccccccc}
        \toprule
        \multirow{2.6}{*}{Training} & \multirow{2.6}{*}{$\theta$} &\multirow{2.6}{*}{NAC} & \multirow{2.6}{*}{ACR} & \multicolumn{9}{c}{Certified Accuracy at Radius r}\\
		\cmidrule(lr){5-13}
		& & & & 0.0 & 0.25 & 0.5 & 0.75 & 1.0 & 1.25 & 1.5 & 1.75 & 2.00\\
        \midrule
        \multirow{12}{*}{\cohen}
        & 0.00 & 99.0 & 0.000 & 99.0 & 0.0 & 0.0 & 0.0 & 0.0 & 0.0 & 0.0 & 0.0 & 0.0\\
        & 0.10 & 98.6 & 0.118 & 98.6 & 13.0 & 9.2 & 6.6 & 4.6 & 3.0 & 1.4 & 0.8 & 0.0\\
        & 0.20 & 98.2 & 0.187 & 98.2 & 20.2 & 15.2 & 10.4 & 7.6 & 4.8 & 3.0 & 1.2 & 0.0\\
        & 0.30 & 96.8 & 0.264 & 96.6 & 27.6 & 21.0 & 15.8 & 11.0 & 7.2 & 4.4 & 2.0 & 0.0\\
        & 0.40 & 94.6 & 0.347 & 94.0 & 34.6 & 27.2 & 21.0 & 15.8 & 10.8 & 6.4 & 2.8 & 0.0\\
        & 0.50 & 88.8 & 0.422 & 88.6 & 41.2 & 32.4 & 24.8 & 18.8 & 13.8 & 8.0 & 4.0 & 0.0\\
\bottomrule
        \end{tabular}
\end{threeparttable}
}
\vspace{-4mm}
\end{table}

\begin{table}[p]
	\centering
	\small
	\centering
    \caption{Natural accuracy (NAC), average certified radius (ACR) and certified selection rate (portion of samples selected for the certification-network) at different radii on \cifar with $\sigma_{t}=\sigma_{\epsilon}=0.50$ for a range of threshold parameters $\theta$ and an \tool model with entropy selection, a \RNB certification-network and an \LANET core-network.} 
	\label{tab:cifar10_selection_050}
	\scalebox{0.9}{
	\begin{threeparttable}
    \begin{tabular}{ccccccccccccc}
        \toprule
        \multirow{2.6}{*}{Training} & \multirow{2.6}{*}{$\theta$} &\multirow{2.6}{*}{NAC} & \multirow{2.6}{*}{ACR} & \multicolumn{9}{c}{Certified Selection Rate at Radius r}\\
		\cmidrule(lr){5-13}
		& & & & 0.0 & 0.25 & 0.5 & 0.75 & 1.0 & 1.25 & 1.5 & 1.75 & 2.00\\
        \midrule
        \multirow{12}{*}{\cohen}
        & 0.00 & 99.0 & 0.000 & 0.0 & 0.0 & 0.0 & 0.0 & 0.0 & 0.0 & 0.0 & 0.0 & 0.0\\
        & 0.10 & 98.6 & 0.118 & 18.0 & 13.0 & 9.2 & 6.6 & 4.6 & 3.0 & 1.4 & 0.8 & 0.0\\
        & 0.20 & 98.2 & 0.187 & 27.8 & 20.6 & 15.4 & 10.4 & 7.6 & 4.8 & 3.0 & 1.2 & 0.0\\
        & 0.30 & 96.8 & 0.264 & 37.2 & 29.0 & 22.0 & 16.2 & 11.0 & 7.2 & 4.4 & 2.0 & 0.0\\
        & 0.40 & 94.6 & 0.347 & 47.8 & 37.8 & 30.4 & 25.2 & 19.2 & 13.6 & 8.6 & 4.2 & 0.0\\
        & 0.50 & 88.8 & 0.422 & 62.0 & 50.2 & 39.2 & 31.6 & 26.0 & 20.4 & 13.4 & 8.2 & 0.0\\
\bottomrule
        \end{tabular}
\end{threeparttable}
}
\vspace{-4mm}
\end{table}

\begin{table}[p]
	\centering
	\small
	\centering
    \caption{Natural accuracy (NAC), average certified radius (ACR) and certified accuracy at different radii on \cifar with $\sigma_{t}=\sigma_{\epsilon}=1.00$ for a range of threshold parameters $\theta$ and an \tool model with entropy selection, a \RNB certification-network and an \LANET core-network.}  
	\label{tab:cifar10_main_100}
	\scalebox{0.9}{
	\begin{threeparttable}
    \begin{tabular}{ccccccccccccc}
        \toprule
        \multirow{2.6}{*}{Training} & \multirow{2.6}{*}{$\theta$} &\multirow{2.6}{*}{NAC} & \multirow{2.6}{*}{ACR} & \multicolumn{9}{c}{Certified Accuracy at Radius r}\\
		\cmidrule(lr){5-13}
		& & & & 0.0 & 0.5 & 1.0 & 1.5 & 2.0 & 2.5 & 3.0 & 3.5 & 4.0\\
        \midrule
        \multirow{12}{*}{\cohen}
        & 0.00 & 99.0 & 0.000 & 99.0 & 0.0 & 0.0 & 0.0 & 0.0 & 0.0 & 0.0 & 0.0 & 0.0\\
        & 0.10 & 98.8 & 0.020 & 98.8 & 1.4 & 0.8 & 0.2 & 0.0 & 0.0 & 0.0 & 0.0 & 0.0\\
        & 0.20 & 98.8 & 0.054 & 98.8 & 4.2 & 2.4 & 1.0 & 0.2 & 0.0 & 0.0 & 0.0 & 0.0\\
        & 0.30 & 98.4 & 0.093 & 98.2 & 6.2 & 4.4 & 2.4 & 1.0 & 0.2 & 0.0 & 0.0 & 0.0\\
        & 0.40 & 97.0 & 0.155 & 97.4 & 10.8 & 6.6 & 4.0 & 2.4 & 0.6 & 0.2 & 0.0 & 0.0\\
        & 0.50 & 95.4 & 0.239 & 95.4 & 14.8 & 10.0 & 6.6 & 4.0 & 1.8 & 0.4 & 0.0 & 0.0\\
\bottomrule
        \end{tabular}
\end{threeparttable}
}
\vspace{-4mm}
\end{table}

\begin{table}[p]
	\centering
	\small
	\centering
    \caption{Natural accuracy (NAC), average certified radius (ACR) and certified selection rate (portion of samples selected for the certification-network) at different radii on \cifar with $\sigma_{t}=\sigma_{\epsilon}=1.00$ for a range of threshold parameters $\theta$ and an \tool model with entropy selection, a \RNB certification-network and an \LANET core-network.} 
	\label{tab:cifar10_selection_100}
	\scalebox{0.9}{
	\begin{threeparttable}
    \begin{tabular}{ccccccccccccc}
        \toprule
        \multirow{2.6}{*}{Training} & \multirow{2.6}{*}{$\theta$} &\multirow{2.6}{*}{NAC} & \multirow{2.6}{*}{ACR} & \multicolumn{9}{c}{Certified Selection Rate at Radius r}\\
		\cmidrule(lr){5-13}
        & & & & 0.0 & 0.5 & 1.0 & 1.5 & 2.0 & 2.5 & 3.0 & 3.5 & 4.0\\
        \midrule
        \multirow{12}{*}{\cohen}
        & 0.00 & 99.0 & 0.000 & 0.0 & 0.0 & 0.0 & 0.0 & 0.0 & 0.0 & 0.0 & 0.0 & 0.0\\
        & 0.10 & 98.8 & 0.020 & 4.2 & 1.4 & 0.8 & 0.2 & 0.0 & 0.0 & 0.0 & 0.0 & 0.0\\
        & 0.20 & 98.8 & 0.054 & 6.4 & 4.4 & 2.6 & 1.0 & 0.2 & 0.0 & 0.0 & 0.0 & 0.0\\
        & 0.30 & 98.4 & 0.093 & 11.4 & 6.4 & 4.6 & 2.6 & 1.0 & 0.2 & 0.0 & 0.0 & 0.0\\
        & 0.40 & 97.0 & 0.155 & 17.6 & 11.8 & 7.0 & 4.2 & 2.8 & 0.6 & 0.2 & 0.0 & 0.0\\
        & 0.50 & 95.4 & 0.239 & 26.6 & 17.8 & 11.8 & 7.4 & 4.6 & 2.2 & 0.4 & 0.2 & 0.0\\
\bottomrule
        \end{tabular}
\end{threeparttable}
}
\vspace{-4mm}
\end{table}

\begin{figure}[t]
	\centering
	\scalebox{0.97}{
	\makebox[\textwidth][c]{
        \begin{subfigure}[t]{0.32\textwidth}
			\centering
			\includegraphics[width=\textwidth]{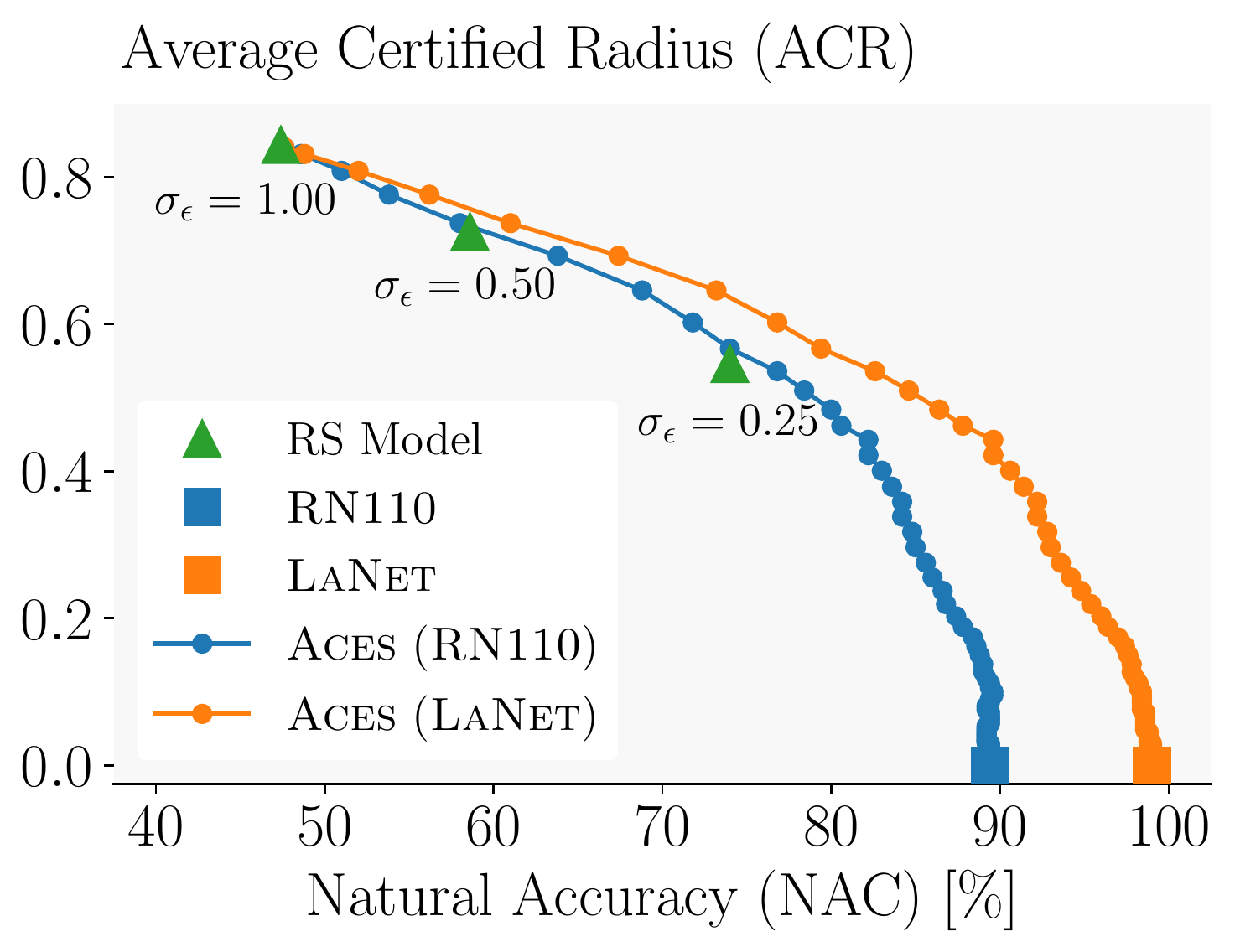}
		\end{subfigure}
		\hfill
		\begin{subfigure}[t]{0.32\textwidth}
			\centering
			\includegraphics[width=\textwidth]{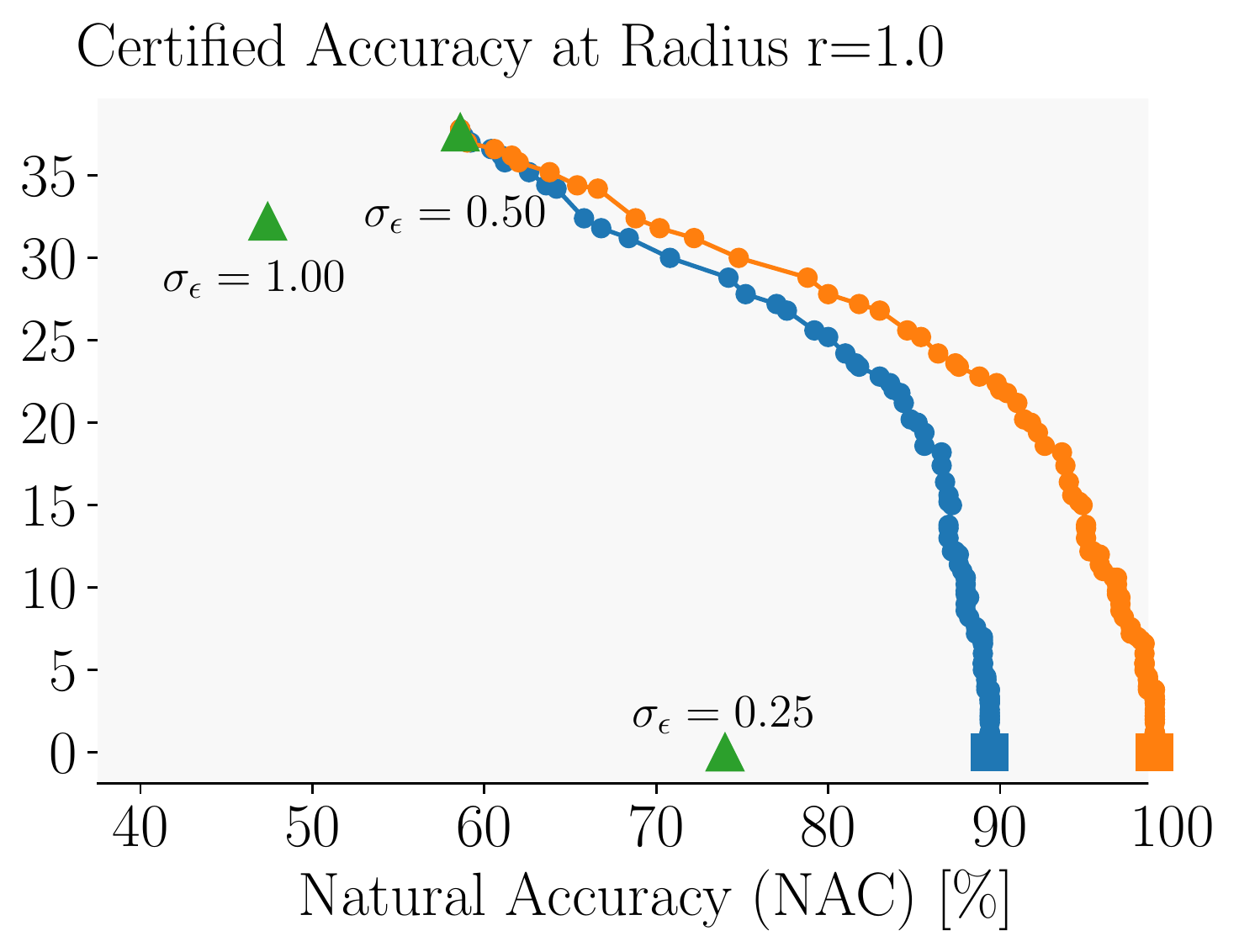}
		\end{subfigure}
		\hfill
		\begin{subfigure}[t]{0.32\textwidth}
			\centering
			\includegraphics[width=\textwidth]{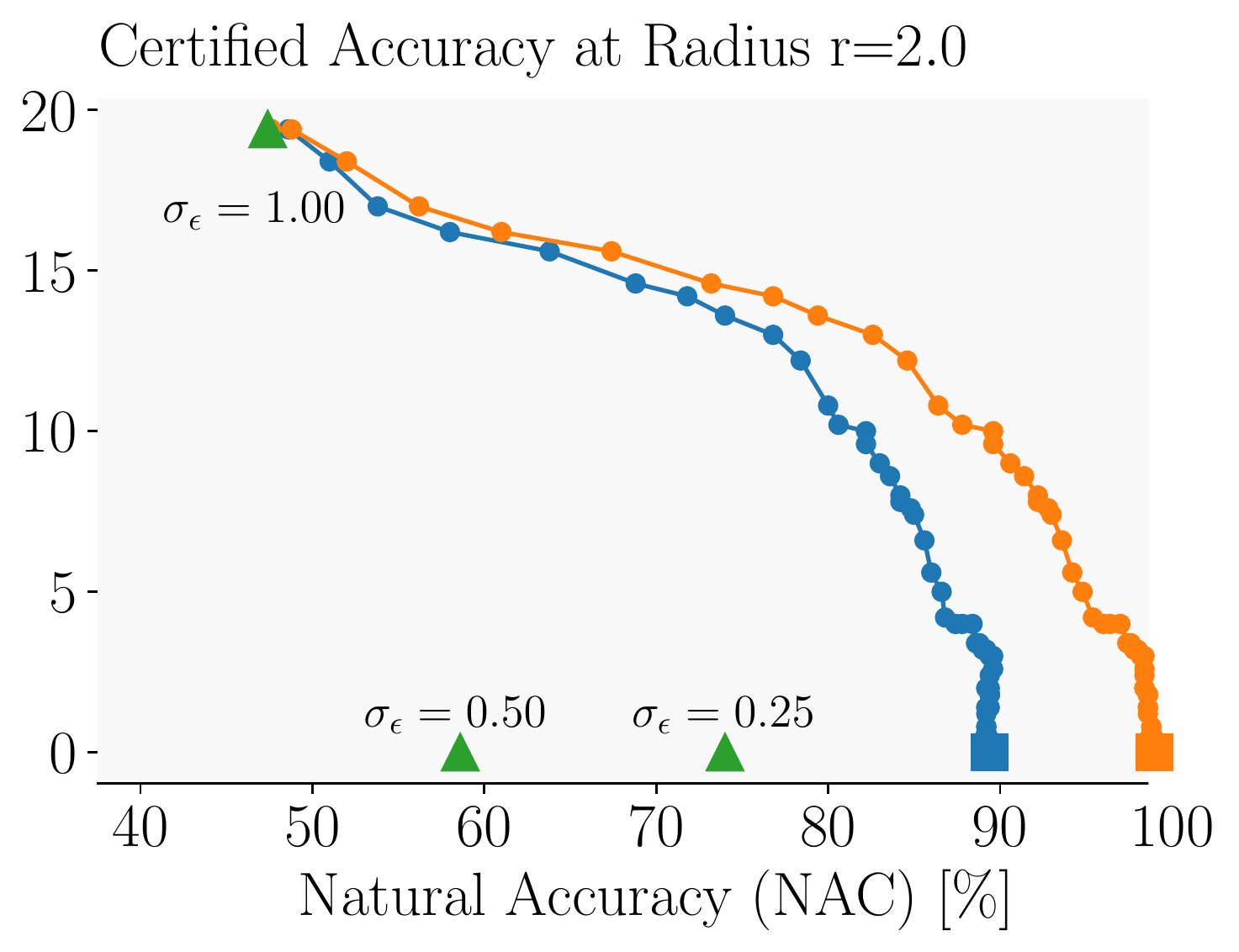}
		\end{subfigure}
	}
	}
	\vspace{-3mm}
	\caption{Comparison of \tool (blue and orange dots) and individual smoothed models (green triangles) on \cifar with \smoothadv trained models with respect to average certified radius (left), certified accuracy at $r=1.0$ (middle), and certified accuracy at $r=2.0$ (right) over natural accuracy. We use \RNB for individual networks and as certification-networks for all \tool models. We consider \tool models with \RNB and \LANET core-networks.}
	\label{fig:ace_cifar_ca_appendix}
\end{figure}

In this section, we evaluate \tool on \cifar for a wide range of training methods (\cohen, \smoothadv, \macer, and \consistency) and noise levels $\sigma \in \{0.25, 0.50, 1.00\}$.
In particular, we provide detailed results on the certified accuracies obtained by \tool in \cref{tab:cifar10_main_025} and the corresponding certified selection rates in \cref{tab:cifar10_selection_025} for $\sigma_{t} = \sigma_{\epsilon} = 0.25$.
Similarly, \cref{tab:cifar10_main_050,tab:cifar10_selection_050} and \cref{tab:cifar10_main_100,tab:cifar10_selection_100} contain results for $\sigma_{\epsilon} = 0.5$ and $\sigma_{\epsilon}=1.0$, respectively.

In \cref{fig:ace_cifar_ca_appendix}, we visualize the trade-off between natural and certified accuracy at fixed radii for \tool (blue and orange dots) and individual smoothed models (green triangles).
We observe that \tool achieves significant certified accuracies at natural accuracies not achievable at all by conventional smoothed models.

\subsection{Selection-Mechanism Ablation}
\label{sec:appendix-selection-ablation}

In this section, we investigate the entropy-based selection-mechanism, introduced in \cref{sec:ace_smoothing}, in more detail and compare it to one based on a separate selection-network.

\subsubsection{Selection Certification}
\label{sec:appendix-selection-certification}

\begin{figure}[t]
	\centering
	\begin{minipage}[t]{.5\textwidth}
	\centering
	\includegraphics[width=0.95\textwidth]{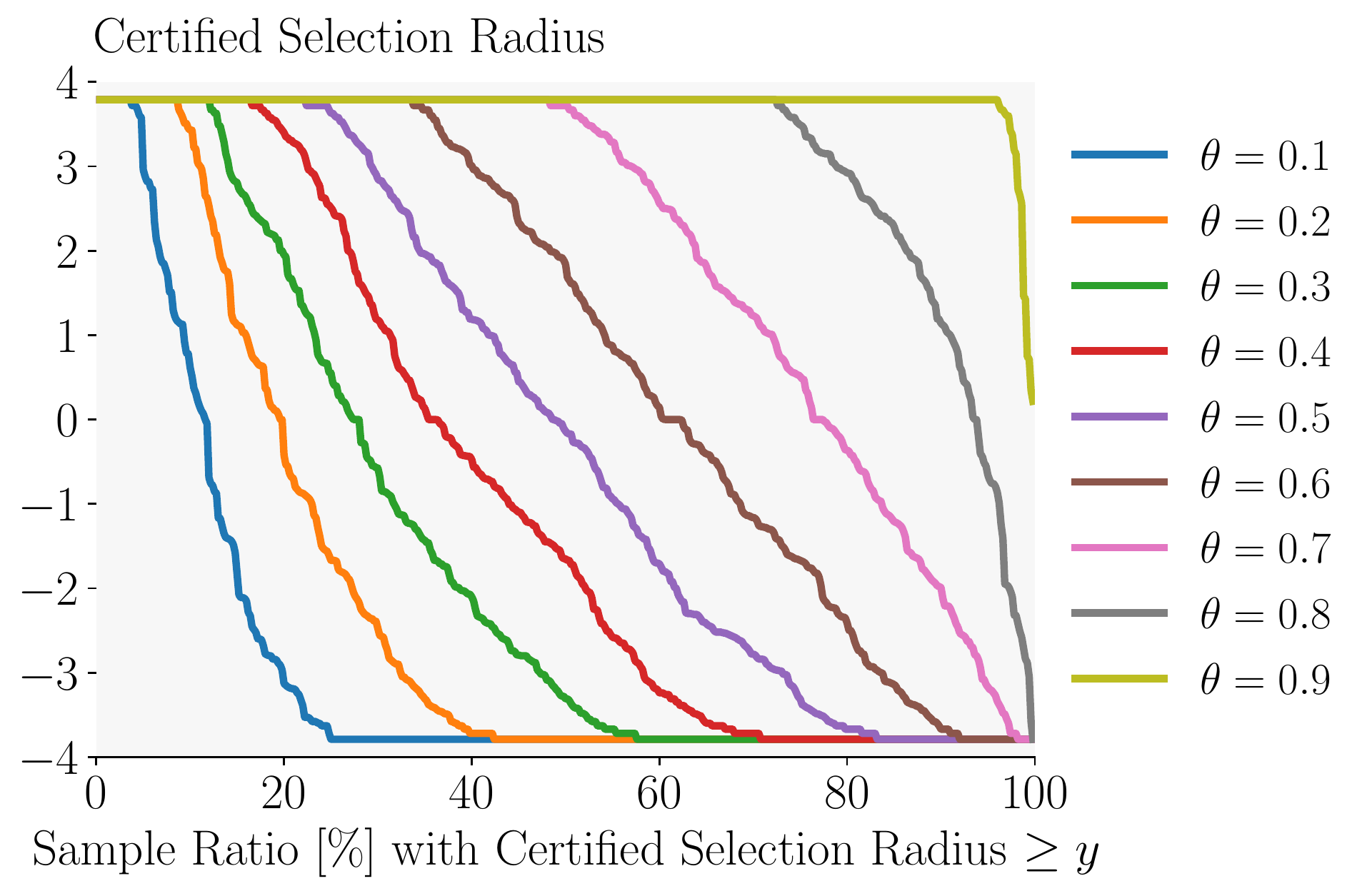} %
	\caption{Certified radii of an entropy-based selection mechanism for a range of $\theta$ over percentile on \IN for a \smoothadv trained \RNM model. Positive radii correspond to the selection of the certification-network and negative radii to that of the core-network. A zero radius corresponds to abstentions \abstain.}
	\label{fig:selection-radii}
\end{minipage}
\hfil
\begin{minipage}[t]{.36\textwidth}
	\centering
	\includegraphics[width=0.9\textwidth]{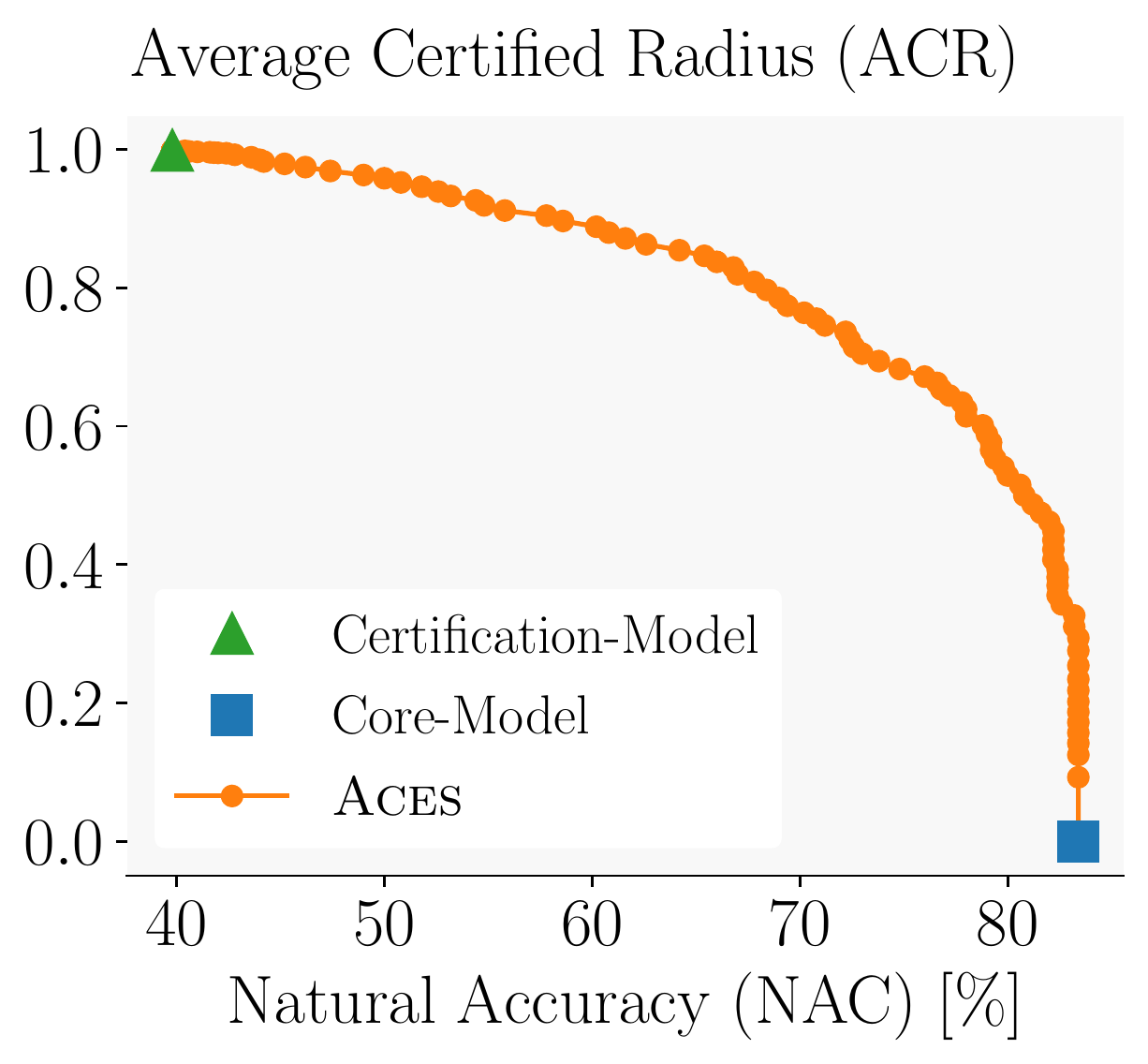} %
	\caption{ACR over natural accuracy for an \tool model based on \smoothadv trained \RNM certification-network, a corresponding entropy-based selection-network and an \EN core-network on \IN.}
	\label{fig:selection-nac-vs-acr}
\end{minipage}
\end{figure}

In \cref{fig:selection-radii}, we visualize the certified radii of the prediction of an entropy-based selection-mechanism based on an \smoothadv trained \RNM with $\sigma = 1.00$ for \IN. 
A positive radius corresponds to a certified selection of the certification-network with that radius, and a negative radius corresponds to a certified selection of the core-network.
A radius of 0 corresponds to the selection-mechanism abstaining.
We generally observe that the selection-mechanism only abstains on very few samples.
Further, for most samples and especially at high or low values of $\theta$, (almost) all perturbations lead to the same selection decision and hence the mathematically maximal certified radius (for a given confidence and sample count).
This is crucial, as the certified radius obtained for \tool is the minimum of those obtained for the certification-network and selection-mechanism.

\subsubsection{Training a Selection Model}
\label{sec:appendix-selection-model}

\begin{figure}[t]
	\centering
	\scalebox{0.97}{
	\makebox[\textwidth][c]{
        \begin{subfigure}[t]{0.38\textwidth}
			\centering
			\includegraphics[width=\textwidth]{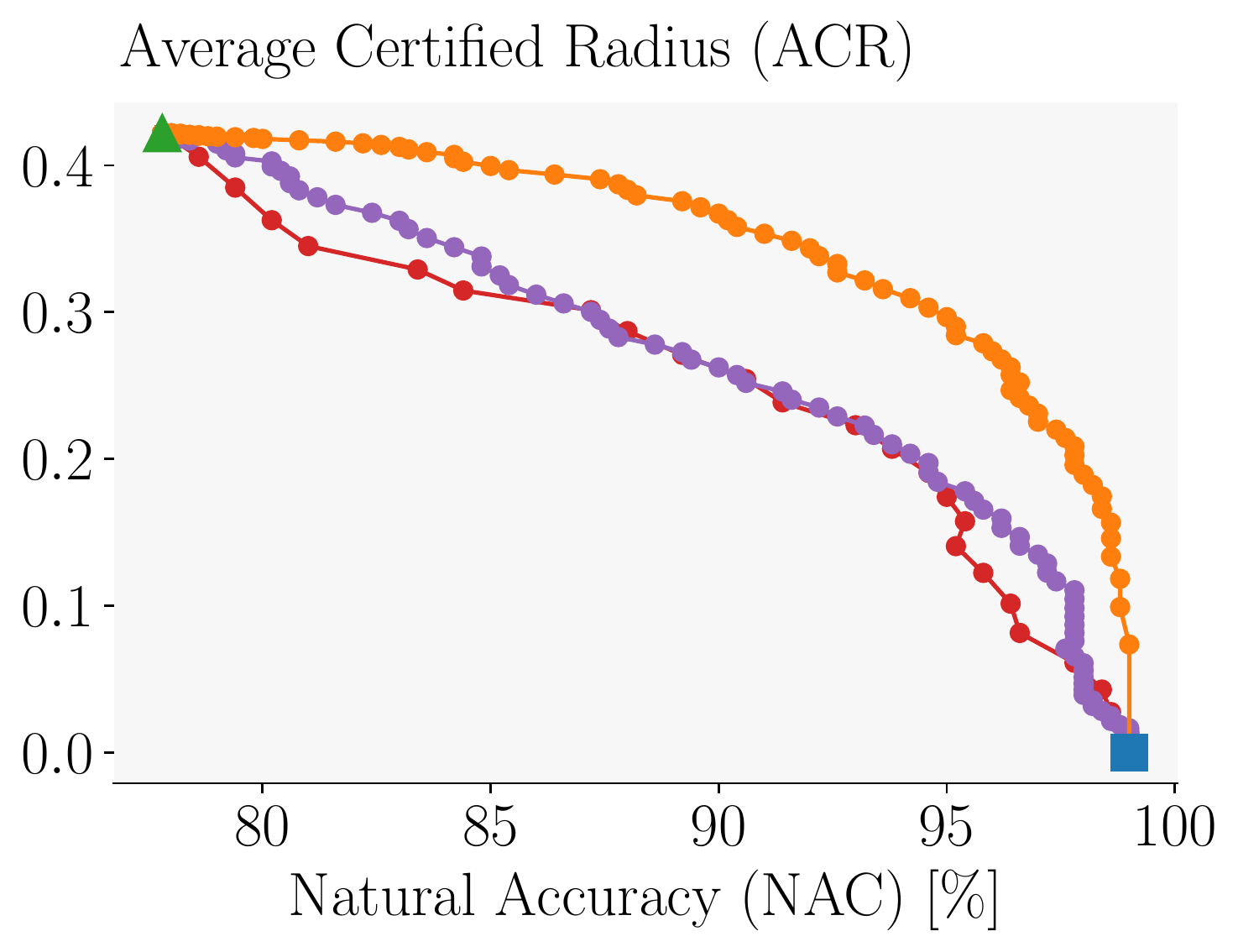}
		\end{subfigure}
		\hfill
		\begin{subfigure}[t]{0.58\textwidth}
			\centering
			\includegraphics[width=\textwidth]{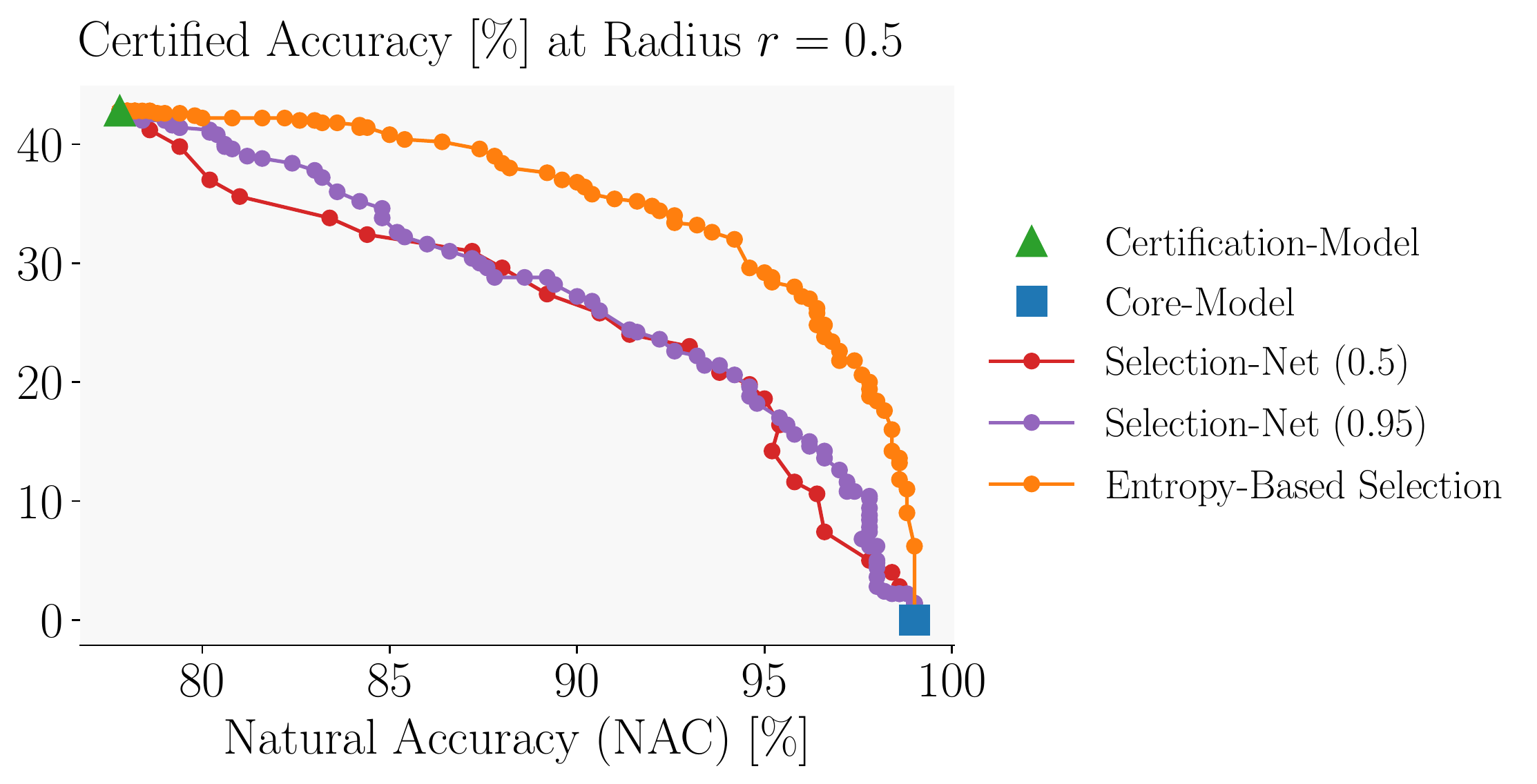}
		\end{subfigure}
    }
	}
	\vspace{-3mm}
	\caption{Comparing \tool with the entropy-based selection mechanism (orange) and selection networks (with $\eta \in \{0.5, 0.95\}$, red and purple) on \cifar with respect to ACR (left) and certified accuracy at radius $r=0.5$ (right) compared to natural accuracy. The certification network (\RNB trained with \cohen, $\sigma=0.25$) and core-network (\LANET) are fixed.}
	\label{fig:selection_net_main_figure}
\end{figure}

\begin{table}[p]
	\centering
	\small
	\centering
    \caption{Comparions of natural accuracy (NAC), average certified radius (ACR) and various certified radii via \tool on \cifar with $\sigma_{\epsilon}=0.25$. We consider selection networks trained with $\eta \in \{0.5, 0.95\}$ for various threshold parameters $\theta$. All selection and certification networks have a \RNB and all core models have a \LANET architecture, and the certification-network was trained with \cohen.}  
	\label{tab:cifar10_selection_net_025}
	\scalebox{0.9}{
	\begin{threeparttable}
    \begin{tabular}{ccccccccccccc}
        \toprule
        \multirow{2.6}{*}{$\eta$} & \multirow{2.6}{*}{$\theta$} &\multirow{2.6}{*}{NAC} & \multirow{2.6}{*}{ACR} & \multicolumn{9}{c}{Certified Accuracy at Radius r}\\
		\cmidrule(lr){5-13}
		& & & & 0.0 & 0.25 & 0.5 & 0.75 & 1.0 & 1.25 & 1.5 & 1.75 & 2.00\\
        \midrule
        \multirow{12}{*}{0.50}   
        & 0.00 & 99.0 & 0.000 & 99.0 & 0.0 & 0.0 & 0.0 & 0.0 & 0.0 & 0.0 & 0.0 & 0.0\\
        & 0.10 & 95.2 & 0.141 & 95.2 & 21.4 & 14.2 & 8.0 & 0.0 & 0.0 & 0.0 & 0.0 & 0.0\\
        & 0.20 & 87.2 & 0.301 & 86.0 & 42.2 & 31.0 & 18.0 & 0.0 & 0.0 & 0.0 & 0.0 & 0.0\\
        & 0.30 & 77.8 & 0.422 & 75.4 & 60.0 & 42.8 & 25.6 & 0.0 & 0.0 & 0.0 & 0.0 & 0.0\\
        & 0.40 & 77.8 & 0.422 & 75.4 & 60.0 & 42.8 & 25.6 & 0.0 & 0.0 & 0.0 & 0.0 & 0.0\\
        & 0.50 & 77.8 & 0.422 & 75.4 & 60.0 & 42.8 & 25.6 & 0.0 & 0.0 & 0.0 & 0.0 & 0.0\\
        & 0.60 & 77.8 & 0.422 & 75.4 & 60.0 & 42.8 & 25.6 & 0.0 & 0.0 & 0.0 & 0.0 & 0.0\\
        & 0.70 & 77.8 & 0.422 & 75.4 & 60.0 & 42.8 & 25.6 & 0.0 & 0.0 & 0.0 & 0.0 & 0.0\\
        & 0.80 & 77.8 & 0.422 & 75.4 & 60.0 & 42.8 & 25.6 & 0.0 & 0.0 & 0.0 & 0.0 & 0.0\\
        & 0.90 & 77.8 & 0.422 & 75.4 & 60.0 & 42.8 & 25.6 & 0.0 & 0.0 & 0.0 & 0.0 & 0.0\\
        & 1.00 & 77.8 & 0.422 & 75.4 & 60.0 & 42.8 & 25.6 & 0.0 & 0.0 & 0.0 & 0.0 & 0.0\\                          
        \midrule
        \multirow{12}{*}{0.95}
        & 0.00 & 99.0 & 0.000 & 99.0 & 0.0 & 0.0 & 0.0 & 0.0 & 0.0 & 0.0 & 0.0 & 0.0\\
        & 0.10 & 98.6 & 0.021 & 98.6 & 3.0 & 2.2 & 1.2 & 0.0 & 0.0 & 0.0 & 0.0 & 0.0\\
        & 0.20 & 98.0 & 0.061 & 98.0 & 9.0 & 6.2 & 2.4 & 0.0 & 0.0 & 0.0 & 0.0 & 0.0\\
        & 0.30 & 97.4 & 0.117 & 97.2 & 17.2 & 10.8 & 5.8 & 0.0 & 0.0 & 0.0 & 0.0 & 0.0\\
        & 0.40 & 95.4 & 0.178 & 95.2 & 26.4 & 17.0 & 9.4 & 0.0 & 0.0 & 0.0 & 0.0 & 0.0\\
        & 0.50 & 91.6 & 0.240 & 91.6 & 35.4 & 24.2 & 14.2 & 0.0 & 0.0 & 0.0 & 0.0 & 0.0\\
        & 0.60 & 87.4 & 0.295 & 86.8 & 41.4 & 30.0 & 18.0 & 0.0 & 0.0 & 0.0 & 0.0 & 0.0\\
        & 0.70 & 83.2 & 0.356 & 81.4 & 49.8 & 37.2 & 21.8 & 0.0 & 0.0 & 0.0 & 0.0 & 0.0\\
        & 0.80 & 80.2 & 0.403 & 78.0 & 56.4 & 41.2 & 25.2 & 0.0 & 0.0 & 0.0 & 0.0 & 0.0\\
        & 0.90 & 77.8 & 0.421 & 75.4 & 59.8 & 42.6 & 25.6 & 0.0 & 0.0 & 0.0 & 0.0 & 0.0\\
        & 1.00 & 77.8 & 0.422 & 75.4 & 60.0 & 42.8 & 25.6 & 0.0 & 0.0 & 0.0 & 0.0 & 0.0\\                                   
        \bottomrule
        \end{tabular}
\end{threeparttable}
}
\vspace{-4mm}
\end{table}

Instead of using an entropy-based selection-mechanism as discussed in \cref{sec:ace_smoothing}, we experimented with following \citet{mueller2021certify} in training a separate binary classifier on this selection task.
To generate the labels, we first sample $n$ perturbed instances of every training input and compute the corresponding prediction by the certification-network and determine the count of correct prediction $n_y$. We then threshold the accuracy of an individual sample over perturbations $n_y/n$ with hyperparameter $\eta$ to obtain the label $\mathbb{I}_{n_y/n >= \eta}$.
We use these labels to then train a binary classifier of the same architecture and using the same training method as for the certification-network.

We instantiate this approach with $n=1000$, $\eta \in \{0.5, 0.95\}$, and \cohen training and compare the obtained \tool models with ones using entropy-based selection in \cref{tab:cifar10_selection_net_025}, visualized in \cref{fig:selection_net_main_figure}.
We observe that the entropy-based selection performs significantly better across all natural accuracies than this selection-network based approach.
Additionally, the entropy-based mechanism does not need any additional training as it is based on the certification-network.
Therefore, we focus all other analysis on entropy-based selection-mechanisms.

\subsection{Varying Inference Noise Magnitude}
\label{sec:appendix-baselines}

\begin{figure}[t]
	\centering
	\scalebox{0.97}{
	\makebox[\textwidth][c]{
		\begin{subfigure}[t]{0.3675\textwidth}
			\centering
			\includegraphics[width=\textwidth]{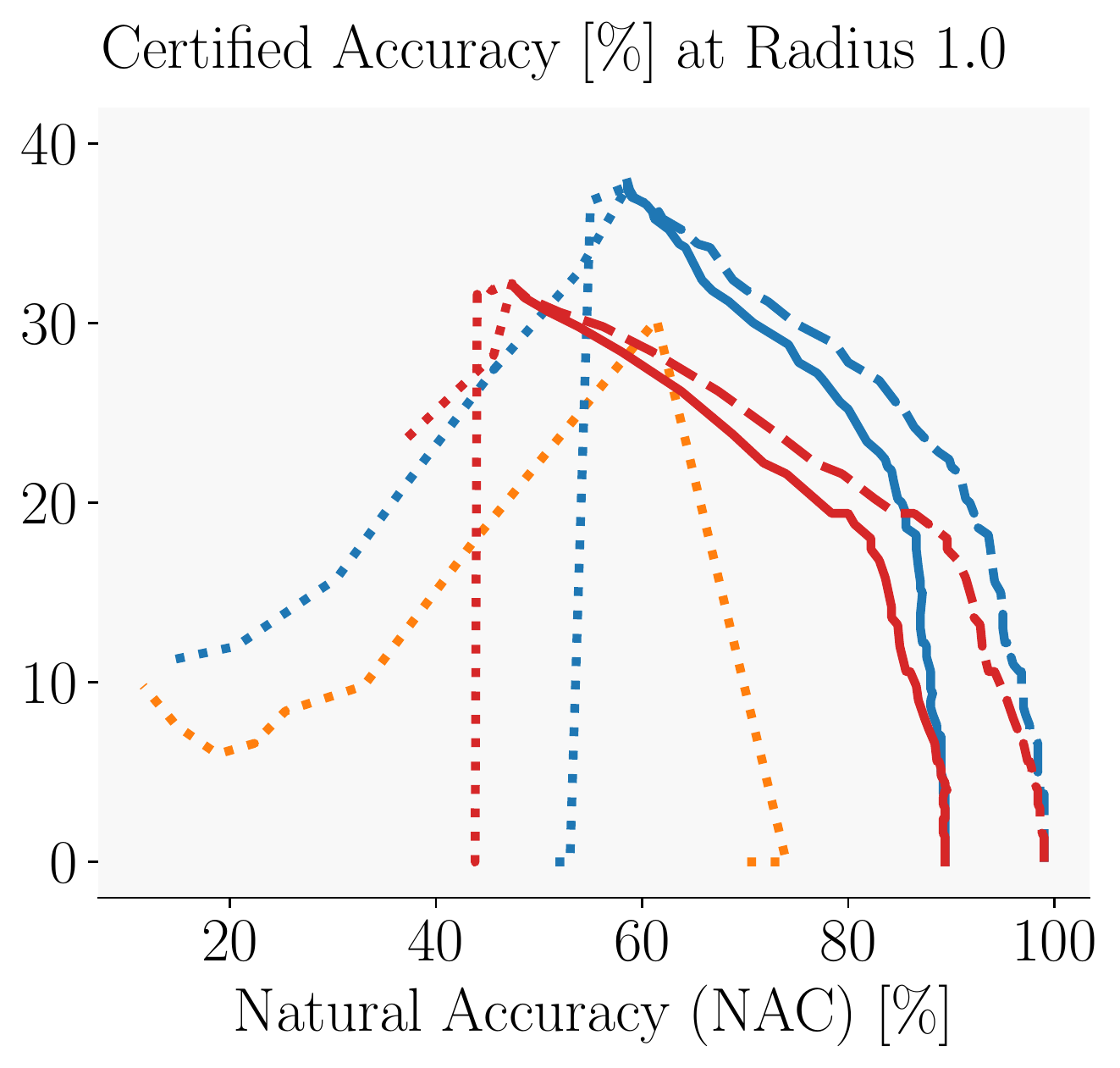}
		\end{subfigure}
		\hfill
		\begin{subfigure}[t]{0.6125\textwidth}
			\centering
			\includegraphics[width=\textwidth]{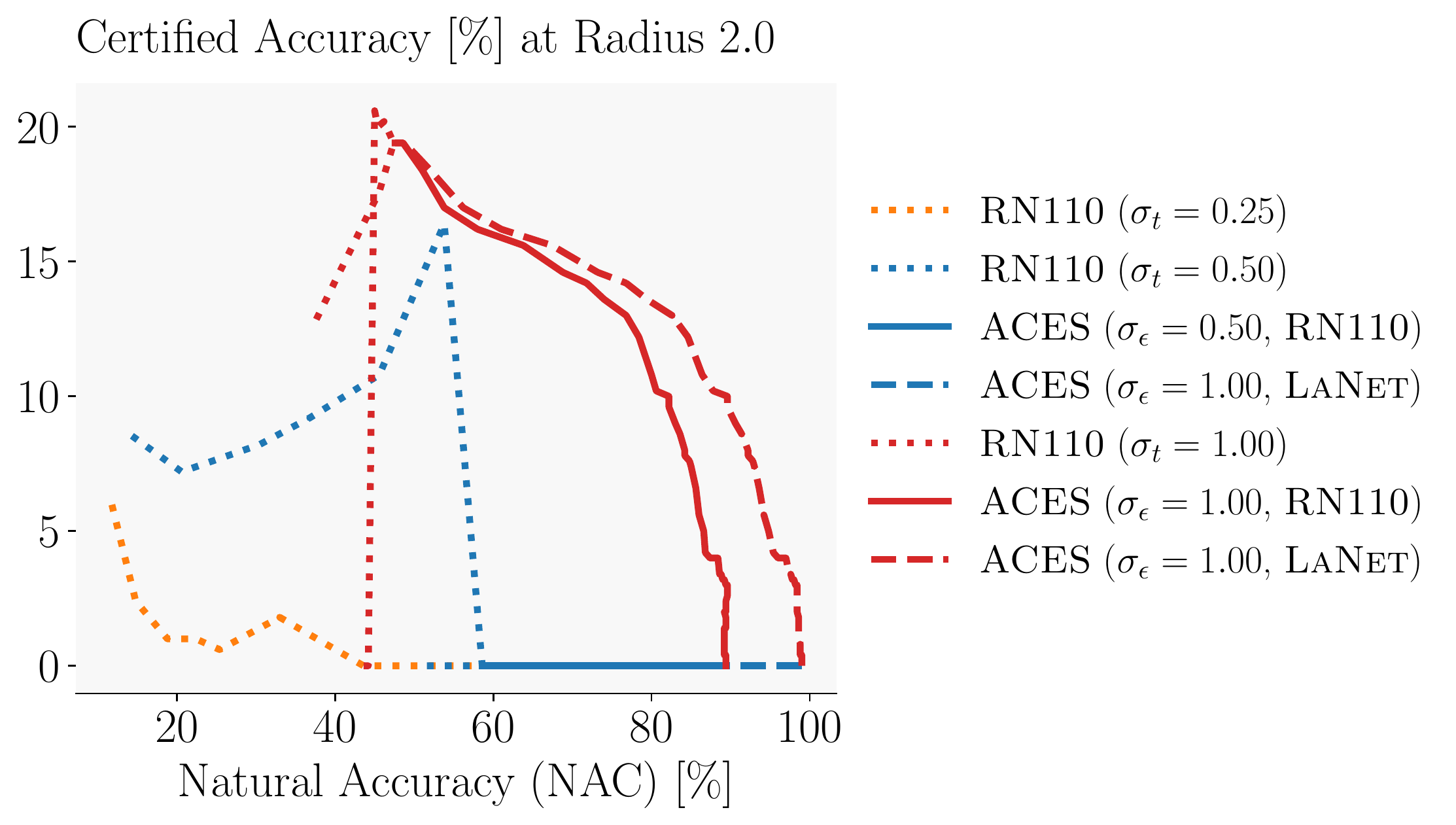}
		\end{subfigure}
	}
	}
	\vspace{-3mm}
	\caption{Certified accuracy at fixed radii over natural accuracy for \tool (solid and dashed lines) and individual smoothed models (dotted lines) using \smoothadv training. We compare stand alone \RNB (dotted lines) trained with $\sigma_t\in \{0.25, 0.5, 1.0\}$ (orange, blue and red) and evaluated on a wide range of $\sigma_{\epsilon} \in [0.0, 1.5]$ with corresponding \tool models evaluated at training noise level and based on \RNB for both certification- and core-networks (solid lines) or a \LANET core-network (dashed lines).}
	\label{fig:ace_smoothadv_cifar_appendix}
\end{figure}

\begin{table}[p]
	\centering
	\small
	\centering
	\caption{Varying the evaluation noise magnitude $\sigma_{\epsilon}$ for a \RNB trained using \smoothadv and $\sigma_{t} \in \{0.25, 0.5\}$ or natural training $\sigma_t=0.0$ on \cifar.} 
	\label{tab:cifar10_baseline}
	\scalebox{0.9}{
	\begin{threeparttable}
    \begin{tabular}{ccccccccccccc}
        \toprule
        \multirow{2.6}{*}{Training $\sigma_{t}$} & \multirow{2.6}{*}{Evaluation $\sigma_{\epsilon}$} &\multirow{2.6}{*}{NAC} & \multirow{2.6}{*}{ACR} & \multicolumn{9}{c}{Radius r}\\
		\cmidrule(lr){5-13}
		& & & & 0.0 & 0.25 & 0.5 & 0.75 & 1.0 & 1.25 & 1.5 & 1.75 & 2.0\\
        \midrule
        \multirow{11}{*}{0.00}
        & 0.125 & \textbf{21.2} & 0.038 & \textbf{18.0} & 7.4 & 0.0 & 0.0 & 0.0 & 0.0 & 0.0 & 0.0 & 0.0\\
        & 0.25 & 13.0 & 0.027 & 13.0 & 5.4 & 0.4 & 0.0 & 0.0 & 0.0 & 0.0 & 0.0 & 0.0\\
        & 0.375 & 10.0 & 0.060 & 10.0 & \textbf{9.6} & 7.8 & 1.8 & 0.0 & 0.0 & 0.0 & 0.0 & 0.0\\
        & 0.5 & 11.0 & 0.030 & 11.4 & 6.0 & 1.0 & 0.0 & 0.0 & 0.0 & 0.0 & 0.0 & 0.0\\
        & 0.625 & 7.4 & 0.025 & 7.0 & 4.8 & 1.6 & 0.0 & 0.0 & 0.0 & 0.0 & 0.0 & 0.0\\
        & 0.75 & 8.4 & 0.047 & 8.6 & 7.0 & 5.2 & 1.8 & 0.6 & 0.0 & 0.0 & 0.0 & 0.0\\
        & 0.875 & 9.2 & 0.084 & 9.2 & 9.0 & 7.8 & 6.4 & 4.4 & 1.4 & 0.2 & 0.0 & 0.0\\
        & 1.0 & 9.2 & 0.165 & 9.2 & 9.2 & 9.2 & 9.2 & 9.0 & 8.6 & 7.0 & 6.2 & 2.8\\
        & 1.25 & 9.6 & 0.207 & 9.6 & \textbf{9.6} & \textbf{9.6} & 9.4 & 9.0 & 8.8 & 8.0 & 6.8 & 5.2\\
        & 1.5 & 9.6 & \textbf{0.210} & 9.6 & \textbf{9.6} & \textbf{9.6} & \textbf{9.6} & \textbf{9.6} & \textbf{9.2} & \textbf{8.8} & \textbf{8.0} & \textbf{6.4}\\        
        \midrule
        \multirow{11}{*}{0.25}
        & 0.125 & 72.0 & 0.301 & 71.8 & 63.0 & 0.0 & 0.0 & 0.0 & 0.0 & 0.0 & 0.0 & 0.0\\
        & 0.25 & \textbf{74.0} & \textbf{0.546} & \textbf{73.8} & \textbf{66.8} & \textbf{57.2} & \textbf{46.8} & 0.0 & 0.0 & 0.0 & 0.0 & 0.0\\
        & 0.375 & 61.4 & 0.542 & 59.6 & 53.0 & 43.8 & 36.8 & \textbf{30.2} & \textbf{22.2} & 0.0 & 0.0 & 0.0\\
        & 0.5 & 43.6 & 0.379 & 41.6 & 35.2 & 29.4 & 24.4 & 17.8 & 12.0 & 7.8 & 3.0 & 0.0\\
        & 0.625 & 33.0 & 0.250 & 31.2 & 24.2 & 21.0 & 15.0 & 9.8 & 6.6 & 4.4 & 2.4 & 1.8\\
        & 0.75 & 25.4 & 0.191 & 24.0 & 19.6 & 15.0 & 11.6 & 8.4 & 4.2 & 2.4 & 1.6 & 0.6\\
        & 0.875 & 22.4 & 0.164 & 19.8 & 17.2 & 13.0 & 9.8 & 6.6 & 4.8 & 2.0 & 1.2 & 1.0\\
        & 1.0 & 18.8 & 0.154 & 17.4 & 15.6 & 11.8 & 8.8 & 6.0 & 4.6 & 2.4 & 1.0 & 1.0\\
        & 1.25 & 14.8 & 0.169 & 13.6 & 12.0 & 11.4 & 8.2 & 7.6 & 6.2 & 5.0 & 4.4 & 2.4\\
        & 1.5 & 11.6 & 0.233 & 11.2 & 10.4 & 10.2 & 10.2 & 9.8 & 9.6 & \textbf{8.2} & \textbf{7.4} & \textbf{6.2}\\        
        \midrule
        \multirow{11}{*}{0.50}
        & 0.125 & 52.0 & 0.224 & 52.0 & 46.6 & 0.0 & 0.0 & 0.0 & 0.0 & 0.0 & 0.0 & 0.0\\
        & 0.25 & 53.0 & 0.416 & 52.2 & 47.6 & 43.6 & 39.2 & 0.0 & 0.0 & 0.0 & 0.0 & 0.0\\
        & 0.375 & 55.0 & 0.589 & 54.6 & 48.8 & 45.6 & 40.0 & 36.8 & 32.0 & 0.0 & 0.0 & 0.0\\
        & 0.5 & \textbf{58.6} & 0.726 & \textbf{57.2} & \textbf{50.6} & \textbf{45.8} & \textbf{42.4} & \textbf{37.6} & \textbf{32.2} & \textbf{27.8} & \textbf{22.6} & 0.0\\
        & 0.625 & 53.8 & \textbf{0.729} & 52.4 & 49.2 & 44.4 & 39.2 & 32.8 & 28.2 & 24.8 & 20.6 & \textbf{16.4}\\
        & 0.75 & 45.6 & 0.599 & 43.4 & 38.0 & 35.6 & 31.4 & 27.4 & 22.6 & 18.6 & 15.0 & 10.8\\
        & 0.875 & 36.8 & 0.473 & 33.8 & 31.4 & 26.8 & 24.8 & 20.8 & 16.2 & 13.6 & 10.8 & 9.2\\
        & 1.0 & 30.4 & 0.390 & 28.8 & 25.2 & 21.4 & 18.4 & 15.8 & 12.6 & 10.6 & 9.4 & 8.2\\
        & 1.25 & 20.6 & 0.325 & 18.6 & 16.8 & 15.6 & 12.8 & 12.0 & 10.6 & 9.4 & 8.0 & 7.2\\
        & 1.5 & 14.0 & 0.334 & 12.8 & 12.2 & 11.8 & 11.8 & 11.2 & 10.6 & 9.8 & 9.2 & 8.6\\           
        \bottomrule
        \end{tabular}
\end{threeparttable}
}
\vspace{-4mm}
\end{table}

\begin{table}[p]
	\centering
	\small
	\centering
	\caption{Varying the evaluation noise magnitude $\sigma_{\epsilon}$ for a \smoothadv trained \RNB with $\sigma_{t}=1.0$ on \cifar.} 
	\label{tab:cifar10_baseline_100}
	\scalebox{0.9}{
	\begin{threeparttable}
    \begin{tabular}{ccccccccccccc}
        \toprule
        \multirow{2.6}{*}{Training $\sigma_{t}$} & \multirow{2.6}{*}{Evaluation $\sigma_{\epsilon}$} &\multirow{2.6}{*}{NAC} & \multirow{2.6}{*}{ACR} & \multicolumn{9}{c}{Radius r}\\
		\cmidrule(lr){5-13}
		& & & & 0.0 & 0.5 & 1.0 & 1.5 & 2.0 & 2.5 & 3.0 & 3.5 & 4.0\\
        \midrule
        \multirow{11}{*}{1.00}
        & 0.125 & 43.6 & 0.193 & 43.6 & 0.0 & 0.0 & 0.0 & 0.0 & 0.0 & 0.0 & 0.0 & 0.0\\
        & 0.25 & 43.8 & 0.359 & 43.6 & 37.2 & 0.0 & 0.0 & 0.0 & 0.0 & 0.0 & 0.0 & 0.0\\
        & 0.375 & 44.0 & 0.501 & 43.8 & 37.8 & 31.6 & 0.0 & 0.0 & 0.0 & 0.0 & 0.0 & 0.0\\
        & 0.5 & 44.2 & 0.621 & 43.4 & 38.4 & 31.8 & 26.2 & 0.0 & 0.0 & 0.0 & 0.0 & 0.0\\
        & 0.625 & 45.0 & 0.716 & 44.0 & 39.6 & 32.0 & \textbf{27.0} & \textbf{20.6} & 0.0 & 0.0 & 0.0 & 0.0\\
        & 0.75 & 45.4 & 0.787 & 44.8 & \textbf{39.8} & 31.8 & 27.0 & 20.0 & 15.2 & 0.0 & 0.0 & 0.0\\
        & 0.875 & 46.2 & 0.832 & \textbf{45.2} & 38.6 & 32.0 & 26.6 & 20.2 & \textbf{15.6} & 11.0 & 0.0 & 0.0\\
        & 1.0 & \textbf{47.4} & \textbf{0.844} & \textbf{45.2} & 38.0 & \textbf{32.2} & 25.0 & 19.4 & 14.8 & \textbf{11.4} & \textbf{7.4} & 0.0\\
        & 1.25 & 45.6 & 0.762 & 42.2 & 33.8 & 28.2 & 22.2 & 17.6 & 13.0 & 9.4 & 4.8 & \textbf{2.6}\\
        & 1.5 & 37.2 & 0.597 & 33.2 & 29.0 & 23.6 & 18.0 & 12.6 & 9.2 & 5.6 & 3.4 & 1.4\\                
        \bottomrule
        \end{tabular}
\end{threeparttable}
}
\vspace{-4mm}
\end{table}

\begin{table}[t]
	\centering
	\small
	\centering
	\caption{Varying training and inference noise magnitude $\sigma$ for individual \consistency trained \RNM on \IN.} 
	\label{tab:imagenet_baseline_consistency}
	\scalebox{0.9}{
	\begin{threeparttable}
    \begin{tabular}{cccccccccccc}
        \toprule
        \multirow{2.6}{*}{$\sigma_{\epsilon}$} &\multirow{2.6}{*}{NAC} & \multirow{2.6}{*}{ACR} & \multicolumn{9}{c}{Radius r}\\
		\cmidrule(lr){4-12}
		& & & 0.0 & 0.5 & 1.0 & 1.5 & 2.0 & 2.5 & 3.0 & 3.5 & 4.0\\
        \midrule
        0.25 & 63.6 & 0.512 & 63.0 & 54.0 & 0.0 & 0.0 & 0.0 & 0.0 & 0.0 & 0.0 & 0.0\\
        0.50 & 57.0 & 0.806 & 55.4 & 48.8 & 42.2 & 34.8 & 0.0 & 0.0 & 0.0 & 0.0 & 0.0\\
        1.00 & 45.6 & 1.023 & 43.2 & 39.6 & 35.0 & 29.4 & 24.4 & 22.0 & 16.6 & 13.4 & 0.0\\             
        \bottomrule
        \end{tabular}
\end{threeparttable}
}
\vspace{-4mm}
\end{table}

Randomized smoothing is based on perturbing the inputs passed to an underlying model with random noise terms $\epsilon$.
Varying the magnitude of this noise is a natural way to trade-off robustness and accuracy, considered here as a baseline.

We first vary the evaluation noise level $\sigma_{\epsilon}$ and training noise level $\sigma_{t}$ separately for \smoothadv trained \RNB on \cifar and observe that the best ACR is achieved when evaluating a model at (or close to) the noise magnitude it was trained with (see \cref{tab:cifar10_baseline_100,tab:cifar10_baseline}). 
In \cref{fig:ace_smoothadv_cifar}, we illustrate a direct comparison of the thus obtained certified accuracies (dotted lines) with those of \tool models for \RNB (solid lines) and \EN (dashed lines) core-networks.
We generally observe that a) models trained with $\sigma_{t}$ performs best with evaluation noise $\sigma_{e} \approx \sigma_{t}$ in all settings, except where $\sigma_{t}$ is too small to mathematically allow for certification, and b) that reducing the inference noise magnitude often does not improve natural accuracy in sharp contrast to \tool models where much higher natural accuracies can be reached.

Based on this insight and due to the higher computational cost, we vary training and evaluation noise level $\sigma$ jointly for \IN using \consistency training and show results in \cref{tab:imagenet_baseline_consistency}.
Again, we observe that \tool models (orange and blue dots) outperform the thus obtained individual smoothed models (green triangles), reaching natural accuracies far beyond what individual smoothed models can, as is illustrated in \cref{fig:ace_imagenet_ca_appendix}.
Only when purely optimizing for certified accuracy by setting $\theta=1.0$ is \tool outperformed by individual models, as the needed Bonferroni correction increases the required confidence leading to a slight drop in ACR from $0.512, 0.806$, and $1.023$ to $0.509, 0.800$, and $0.997$ for $\sigma_{\epsilon}=0.25, 0.5$, and $1.00$, respectively.

}{}

\message{^^JLASTPAGE \thepage^^J}

\end{document}